\documentclass[runningheads]{llncs}
\usepackage[utf8]{inputenc}
\usepackage{graphicx}
\usepackage{mathptmx}
\usepackage{amsmath}
\usepackage{amssymb}
\usepackage{algorithm}
\usepackage[noend]{algpseudocode}
\usepackage{subcaption}
\newcommand{\R}{\mathbb{R}}
\newcommand{\E}{\mathbb{E}}
\newcommand{\V}{\Vert}
\newcommand{\blank}{\hspace{0.1cm}}
\allowdisplaybreaks
\usepackage{geometry}
\usepackage{cite}
\usepackage[misc]{ifsym}
\begin{document}
\title{Compositional Stochastic Average Gradient for \\Machine Learning and Related Applications}
%
%
\author{Tsung-Yu Hsieh$^{1}$ \and 
Yasser EL-Manzalawy$^2$ \and 
Yiwei Sun$^1$ \and
Vasant Honavar$^{1,2}$
}
\authorrunning{T. Hsieh et al.}
%
\institute{$^1$ Department of Computer Science and Engineering,\\
$^2$ College of Information Science and Technology,\\
The Pennsylvania State University, University Park PA 16802, USA\\
\email{\{tuh45,yme2,yus162,vuh14\}@psu.edu}}
\maketitle              
\begin{abstract}
Many machine learning, statistical inference, and portfolio optimization problems require minimization of a composition of expected value functions (CEVF). Of particular interest is the finite-sum versions of such compositional optimization problems (FS-CEVF). Compositional stochastic variance reduced gradient (C-SVRG) methods that combine stochastic compositional gradient descent (SCGD) and stochastic variance reduced gradient descent (SVRG) methods are  the state-of-the-art methods for FS-CEVF problems.  We introduce compositional stochastic average gradient descent (C-SAG) a novel extension of the stochastic average gradient method (SAG) to minimize composition of finite-sum functions. C-SAG, like SAG, estimates gradient by incorporating memory of previous gradient information. We present theoretical analyses of C-SAG which show that C-SAG, like SAG, and C-SVRG, achieves a linear convergence rate when the objective function is strongly convex; However, C-CAG achieves lower oracle query complexity per iteration than C-SVRG. Finally, we present results of experiments showing that C-SAG converges substantially faster than full gradient (FG), as well as C-SVRG. 

\keywords{Machine Learning \and Stochastic Gradient Descent \and Compositional Finite-sum Optimization \and Stochastic Average Gradient Descent  \and Compositional Stochastic Gradient Descent \and Convex Optimization}
\end{abstract}
\noindent{\bf Note to the readers:} The short version of this paper has been accepted by the 19th International Conference on Intelligent Data Engineering and Automated Learning, 2018. This version includes detailed proofs of the main theorem and the supporting lemmas that are not included in the conference publication due to space constraints.
\section{Introduction}

Many problems in machine learning and statistical inference (e.g., \cite{dai2016learning}), risk management (e.g., \cite{dentcheva2017statistical}), multi-stage stochastic programming (e.g., \cite{shapiro2009lectures}) and adaptive simulation (e.g., \cite{hu2014model}) require minimization of linear or non-linear compositions of expected value functions  (CEVF) \cite{wang2017stochastic}.  Of particular interest are finite-sum versions of the CEVF optimization problems (FS-CEVF for short) which find applications in estimating sparse additive models (SpAM) \cite{SpAM}, maximizing softmax likelihood functions \cite{fagan2018unbiased} (with a wide range of applications in supervised learning \cite{bishop2006pattern,friedman2001elements,theodoridis2015machine}), portfolio optimization \cite{lian2017finite}, and policy evaluation in reinforcement learning \cite{sutton1998reinforcement}. 

Stochastic gradient descent (SGD) methods \cite{kiefer1952stochastic,robbins1985stochastic}, fast first-order methods for optimization of differentiable convex functions, and their variants, have long found applications in machine learning and related areas \cite{rumelhart1986learning,lecun2012efficient,darken1990fast,bottou2010large,amari1993backpropagation,bottou1991stochastic,baird1999gradient,kingma2014adam,sutton2009fast,zeiler2012adadelta,schmidt2012,zhang2004solving,cauwenberghs1993fast,le2011optimization,tan2016barzilai,rakhlin2011making,lecun2015deep,bottou2018optimization}. SGD methods continue to be an area of active research focused on methods for their parallelization \cite{zhang2004solving,recht2011hogwild,zhao2016fast,johnson2013accelerating,jain2018accelerating,wang2016accelerating,schmidt2017minimizing,zinkevich2010parallelized}, and theoretical guarantees \cite{rakhlin2011making,shamir2016convergence,mandt2017stochastic,needell2014stochastic,jin2016provable}.

CEVF and FS-CEVF optimization problems have been topics of focus in several recent papers.  Because SGD methods update the parameter iteratively by using {\em unbiased samples} (queries) of the gradient, the  linearity of the objective function in the sampling probabilities is crucial for establishing the attractive properties of SGD. Because the CEVF objective functions violate the required linearity property, classical SGD methods no longer suffice \cite{wang2017stochastic}. Against this background, Wang et al. \cite{wang2017stochastic}, introduced a class of stochastic compositional gradient descent (SCGD) algorithms for solving CEVF optimization problems. SCGD, a compositional version of stochastic quasi-gradient methods \cite{ermoliev1988stochastic}, has been proven to achieve convergence rate that is sub-linear in $K$, the number of stochastic samples (i.e. iterations)\cite{wang2017stochastic,wang2016accelerating,wang2017accelerating}. 

Lian et al. \cite{lian2017finite} proposed compositional variance-reduced gradient methods (C-SVRG) for FS-CEVF optimization problems. C-SVRG methods combine SCGD methods for CEVF optimization \cite{wang2017stochastic} with stochastic variance-reduced gradient (SVRG) method \cite{johnson2013accelerating} to achieve 
a convergence rate that is constant linear in $K$, the number of stochastic samples, when the FS-CEVF to be optimized is strongly convex \cite{lian2017finite}. Recently, several extensions and variants of the C-SVRG method have been proposed. Liu et al. \cite{liu2017variance} investigated compositional variants of SVRG methods  that can handle non-convex compositions of objective functions that include inner and outer finite-sum functions.
Yu et al. \cite{yufast} incorporated the alternating direction method of multipliers (ADMM) method \cite{boyd2011distributed} into C-SVRG to obtain com-SVR-ADMM to accommodate compositional objective functions with linear constraints. They also showed  that com-SVR-ADMM has a linear convergence rate when the compositional objective function is strongly convex and Lipschitz smooth. Others \cite{huo2017accelerated,lin2018improved} have incorporated proximal gradient descent techniques into C-SVRG to handle objective functions that include a non-smooth convex regularization penalty. Despite the differences in motivations behind the various extensions of C-SVRG and their mathematical formulations, the query complexity per iteration of gradient-based update remains unchanged from that of C-SVRG. Specifically, C-SVRG methods make two oracle queries for every targeted component function to estimate gradient per update iteration.

In light of the preceding observations, it is tempting to consider compositional extensions of stochastic average gradient (SAG) \cite{schmidt2017minimizing,schmidt2012} which offers an attractive alternative to SVRG. SAG achieves $O(1/K)$ convergence rate when the objective function is convex and a linear convergence rate when the objective function is strongly-convex \cite{schmidt2012,schmidt2017minimizing}. While in the general setting, the memory requirement of SAG  is $O(np)$ where $n$ is the number of data samples and $p$ is the dimensionality of the space, in many cases, it can be reduced to $O(n)$ by exploiting problem structure \cite{schmidt2017minimizing}. We introduce the compositional stochastic average gradient (C-SAG) method to minimize composition of finite-sum functions. C-SAG is a natural extension of SAG to the compositional setting. Like SAG, C-SAG estimates gradient by incorporating memory of previous gradient information. We present theoretical analyses of C-SAG which show that C-SAG, like SAG, and C-SVRG, achieves a linear convergence rate when the objective function is strongly convex. However, C-SAG achieves lower oracle query complexity per iteration than C-SVRG. We present results of experiments  showing that C-SAG converges substantially faster than full gradient (FG), as well as C-SVRG.

The rest of the paper is organized as follows. Section 2 introduces the FS-CEVF optimization problem with illustrative examples. Section 3 describes the proposed C-SAG algorithm. Section 4 establishes the  convergence rate of C-SAG. Section 5 presents results of experiments that compare  C-SAG with C-SVRG. Section 6 concludes with a summary and discussion.

\section{Finite-Sum Composition of Expected Values Optimization}
In this section, we introduce some of the key definitions and illustrate how some machine learning problems naturally lead to FS-CEVF optimization problems. 

\subsection{Problem Formulation}
A {\bf finite-sum composition of expected value function}  (FS-CEVF) optimization problem \cite{lian2017finite} takes the form: 
\begin{equation}
	\label{objFunc}
    \begin{aligned}
		&\min_{x} f(x):=F\circ G(x)=F(G(x))\\
        &G(x)=\frac{1}{m}\sum_{j=1}^m G_j(x), \blank \blank F(y)=\frac{1}{n}\sum_{i=1}^n F_i(y).
	\end{aligned}
\end{equation}
where the inner function, $G:\R^p \rightarrow \R^q$, is the empirical mean of $m$ component functions, $G_j:\R^p \rightarrow \R^q$, and the outer function, $F:\R^q \rightarrow \R$, is the empirical mean of $n$ component functions, $F_i:\R^q \rightarrow \R$. 

\subsection{Example: Statistical Learning}
Estimation of sparse additive models (SpAM) \cite{SpAM} is an important statistical learning problem. Suppose we are given a collection of $d$-dimensional input vectors $x_i=\left( x_{i1}, x_{i2}, \ldots, x_{i,d} \right)^T \in \R ^d$, and associated responses $y_i$. Suppose that $y_i = \alpha + \sum_{j=1}^d h_j(x_{ij})+\epsilon_i$ for each pair $\left( x_i, y_i \right)$. In the equation, $h_j : \R \rightarrow \R$ is a feature extractor  and $\epsilon_i$ is  zero-mean noise. It is customary to set $\alpha$ to zero by first subtracting from each data sample,  the mean of the data samples. The feature extractors $h_j$ can in general be non-linear functions. SpAM estimates the feature extractor functions by solving the following minimization problem:
\begin{equation*}
	\min_{h_j \in \mathbb{H}_j, j=1, \ldots, d} \frac{1}{n}\sum_{i=1}^n \left(y_i - \sum_{j=1}^d h_j(x_i) \right)^2,
\end{equation*}
where $\mathbb{H}_j$ is a pre-defined set of  functions chosen to ensure  that the model is identifiable. In many  machine learning applications, the feasible sets $\mathbb{H}_j$ are usually compact and the feature functions $h_j$'s are continuously differentiable. In such scenario, it is straightforward to formulate the objective function in the form of Eq.(\ref{objFunc}). 

In many practical applications of machine learning, $h_j$'s are assumed to be linear. In this case, the model is simplified to $y_i = w^T x_i + \epsilon_i$, where $w$ is a coefficient vector. Penalized variants of regression models (e.g. \cite{amini2008high}, \cite{lasso}, \cite{yuan2007model}) are often preferred to enhance stability and to induce sparse solutions. Of particular interest is the $l_1$-penalized regression, which is well-known by the LASSO estimator \cite{lasso} which can be written as follows:
\begin{equation*}
	w^* = arg\min_w \frac{1}{n}\sum_{i=1}^n \left(y_i - w^T x_i \right)^2 + \lambda \V w \V_1.
\end{equation*}
The objective function can be formulated as an FS-CEVF problem in Eq.(\ref{objFunc}):
\begin{alignat*}{2}
    &G_j(w) = \left( \lambda\vert w_j \vert, w_j\cdot x_{ij} \right)^T \in \R ^2\\
    &z = \frac{1}{d}\sum_{j=1}^d G_j(w) = \left( \frac{\lambda}{d}\V w \V_1, \frac{1}{d} w^T x_i \right)^T \in \R ^2\\
    &F_i(z) = \left( y_i - d\cdot z_2 \right)^2 + d\cdot z_1\\
    \Rightarrow \blank &f := \frac{1}{n}\sum_{i=1}^n F_i\left( \frac{1}{d} \sum_{j=1}^d G_j(w) \right) = \frac{1}{n}\sum_{i=1}^n \left(y_i - w^T x_i \right)^2 + \lambda \V w \V_1,
\end{alignat*}
where $z_i$ denotes the $i$-th element in vector $z$.

\subsection{Example: Reinforcement Learning}
Policy evaluation in reinforcement learning  \cite{sutton1998reinforcement}  presents an instance of FS-CEVF optimization. Let $S$ be a set of states, $A$ a set of actions, $\pi$: $S \rightarrow A$ a policy that maps states into actions, and $r: S \times A \rightarrow \R$ a reward function, Bellman equation for the value of executing policy  $\pi$ starting in some state $s \in S$ is given by:
\begin{equation*}
	V^\pi (s) = \E_\pi \left\lbrace r_{s,s'} +\gamma V^\pi (s') \vert s \right\rbrace, \forall s, s' \in S,
\end{equation*}
where  $r_{s,s'}$ denotes the reward received upon transitioning from $s$ to $s'$, and $0 \leq \gamma <1$ is a discounting factor.  The goal of reinforcement  learning is to find an  optimal policy $\pi^{\star}$ that maximizes $V^\pi (s)$ $\forall s \in S$. 

The Bellman equation becomes intractable for moderately large $S$. Hence, in practice it is often useful to approximate $V^\pi (s)$ by a suitable parameterized function. For example, $V^\pi (s) \approx \phi_s^T w$ for some $w\in \R^d$, where $\phi_s \in \R^d$ is a $d$-dimensional compact  representation of $s$. In this case, reinforcement learning reduces to finding an optimal $w^*$:
\begin{equation*}
	w^* = arg\min_w \frac{1}{\vert S \vert}\sum_{s}\left( \phi_s^T w - q_{\pi,s'}(w) \right)^2,
\end{equation*}
where $q_{\pi,s'}(w) = \sum_{s'} P^\pi_{ss'}\left( r_{s,s'}+\gamma \phi_{s'}^T w \right)$, $P^\pi_{ss'}$ is the state transition probability from $s$ to $s'$. The objective can be formulated as an FS-CEVF problem in Eq.(\ref{objFunc}):
\begin{alignat*}{2}
	&G_j(w) = \left(w,P^\pi_{sj}\left( r_{s,j}+\gamma \phi_{j}^T w\right)  \right)^T \in \R^{d+1}\\
    &z = \frac{1}{\vert S \vert} \sum_{j=1}^{\vert S \vert} G_j(w) = \left( w, \frac{1}{\vert S \vert}\sum_{j=1}^{\vert S \vert} P^\pi_{sj}\left( r_{s,j}+\gamma \phi_{j}^T w\right)  \right)^T \in \R^{d+1}\\
    &F_i(z) = \left( \phi_i^T z_{1:d} - |S|\cdot z_{d+1} \right)^2\\
    \Rightarrow  \blank &f := \frac{1}{\vert S \vert}\sum_{i=1}^{\vert S \vert} F_i \left( \frac{1}{\vert S \vert} \sum_{j=1}^{\vert S \vert} G_j(w) \right) = \frac{1}{\vert S \vert}\sum_{i=1}^{\vert S \vert}\left( \phi_i^T w - q_{\pi,s'}(w) \right)^2.
\end{alignat*}

\subsection{Example: Mean-Variance Portfolio Optimization}
Portfolio optimization characterizes human investment behavior. For each investor, the goal is to maximize the overall  return on investment while minimizing the investment risk. Given $N$ assets, the objective function for the mean-variance portfolio optimization problem can be specified as:
\begin{equation}
	\max_x \frac{1}{n}\sum_{i=1}^{n}<r_i,x>-\frac{1}{n}\sum_{i=1}^{n}\left(<r_i,x>- \frac{1}{n}\sum_{j=1}^{n}<r_j,x>\right)^2,
\end{equation}
where $r_i \in R^N$ is the reward vector observed at time point $i$, where $i$ ranges from $1,2,...,n$. The investment vector $x$ has dimensionality $R^N$ and is the amount invested in each asset.  To this end,  overall profit is described as the mean value of the inner product of the investment vector and the observed reward vectors. The variance of the return is used as a  measure of risk. To this end,  overall profit is described as the mean value of the inner product of the investment vector and the observed reward vectors. The variance of the return is used as a measure of risk.

The objective can be formulated as an instance of FS-CEVF optimization problem. First, we change the sign of the maximization problem and turn it into a minimization problem. Next, specify the function $G_j(x)$ as:
\begin{equation}
	G_j(x) = \left(
    \begin{tabular}{c}
		$x$\\
        $<r_j,x>$
    \end{tabular}
    \right) \in \R^{N+1} \textnormal{, } j =1,\ldots,n.
\end{equation}
The output is a vector in $R^{N+1}$, where the investments are encoded by first $N$ dimensions and the inner product of the investment vector and the reward vector at time point $j$ is stored in the $(N+1)$-th dimension. Setting $y = \frac{1}{m} \sum_{j=1}^m G_j(x)$, we specify the function $F_i(y)$ as
\begin{equation}
	F_i(y) = -y_{N+1} + \left( <r_i,y_{1:N}> - y_{N+1} \right)^2 \in \R \textnormal{, } i =1,\ldots,n.
\end{equation}
In the equation, $y_{N+1}$ indicates the value of the $(N+1)$-th component of the vector $y$, and $y_{1:N}$ refers to the first to the $N$-th element of the vector $y$.

\section{Compositional Stochastic Average Gradient Method}
In this section, we describe the proposed  {\bf compositional stochastic average gradient (C-SAG)} method for optimizing FS-CEVF objective functions.

\subsection{Stochastic average gradient algorithm}
Because compositional stochastic average gradient (C-SAG) algorithm is an extension of the stochastic average gradient method (SAG) \cite{schmidt2012,schmidt2017minimizing}, we begin by briefly reviewing SAG. SAG designed to optimize the sum of a finite number of smooth optimization functions:
\begin{equation}
	\label{fs_obj}
	\min_x f(x) = \frac{1}{n}\sum_{i=1}^n F_i(x)
\end{equation}
The SAG iterations take the form 
\begin{equation}
	x^{k} = x^{k-1} - \frac{\alpha_k}{n}\sum_{i=1}^n y_i^k,
\end{equation}
where $\alpha_k$ is the learning rate and a random index $i_k$ is selected at each iteration during which we set
\begin{equation}
	y_i^k =
    \begin{aligned}
    	\begin{cases}
    		F_i'(x^{k-1})  &\textnormal{if } i = i_k \\
        	y_i^{k-1}  &\textnormal{otherwise.}
		\end{cases}
    \end{aligned}
\end{equation}
Essentially, SAG maintains in memory, the gradient  with respect to each function in the sum of functions being optimized. At each iteration, the gradient value of only one such function is computed and updated in memory.  The SAG method estimates the overall gradient at each iteration by averaging the gradient values stored in memory. Like stochastic gradient (SG) methods \cite{kiefer1952stochastic,robbins1985stochastic}, the cost of each iteration is independent of the number of functions in the sum. However, SAG has improved  convergence rate compared to conventional SG methods, from $O(1/\sqrt[]{K})$ to $O(1/K)$ in general, and from the sub-linear $O(1/K)$ to a linear convergence rate of the form $O(\rho^K)$ for $\rho < 1$ when the objective function in Eq.(\ref{fs_obj}) is strongly-convex.

\subsection{From SAG to C-SAG}
C-SAG extends SAG to accommodate a finite-sum compositional objective function. Specifically, C-SAG maintains in memory, the items needed for updating the parameter $x^{k-1}$ to $x^{k}$. At each iteration, the relevant indices are randomly selected and the relevant gradients are computed and updated in the memory.

\subsubsection{Notations}
We denote by $z_i$ the $i$-th component of vector $z$. We denote the value of the parameter $x$ at the $k$-th iteration by $x^k$. Given a smooth function $H(x):\R^N \rightarrow \R^M$, $\partial H$ denotes the Jacobian of $H$ defined as:
\[
	\partial H:=\frac{\partial H}{\partial x}=\left(
    \begin{tabular}{c c c}
    	$\frac{\partial [H]_1}{\partial [x]_1}$ &$\cdots$ &$\frac{\partial [H]_1}{\partial [x]_N}$\\
        $\vdots$ &$\ddots$ &$\vdots$\\
        $\frac{\partial [H]_M}{\partial [x]_1}$ &$\cdots$ &$\frac{\partial [H]_M}{\partial [x]_N}$
    \end{tabular}\right).
\]
The value of the Jacobian evaluated at $x_k$ is denoted by $\partial H(x_k)$. For function $h:\R^N \rightarrow \R$, the gradient of $h$ is defined by
\[
	\begin{aligned}
		\nabla h(x)&=\left( \frac{\partial h(x)}{\partial x} \right)^T\\
        			&=\left( \frac{\partial h}{\partial [x]_1}, \cdots ,\frac{\partial h}{\partial [x]_N} \right)^T
	\end{aligned}
\]
The gradient of a composition function $f=F(G(x))$ is given by chain rule:
\begin{equation}
	\label{chainrule}
	\nabla f(x)=\left( \partial G(x) \right)^T\nabla F(G(x)),
\end{equation}
where $\nabla F(G(x))$ stands for the gradient of $F$ evaluated at $G(x)$. Evaluating the gradient requires three key terms, namely, the Jacobian of the inner function $\partial G(\cdot)$, value of the inner function $G(\cdot)$, and the gradient of the outer function $\nabla F(\cdot)$.

We use $\hat{f}$ to denote an estimate of the function $f$, and analogously, $\partial \hat{H}$ and $\nabla \hat{h}(x)$ to denote estimates of Jacobian matrices $\partial H$ and gradient vectors $\nabla h(x)$ respectively. To minimize notational clutter, we use $G_j^k$ to denote $G_j(x^k)$, $\partial G_j^k$ to denote $\partial G_j(x^k)$, and $\nabla F_i^k$ to denote $\nabla F_i(x^k)$. Given a set $A$, we use $\vert A \vert$ to denote the number of elements in the set. We use $\E$ to denote expectation.

\subsubsection{C-SAG Algorithm}

We define memories $J_j$ for storing the Jacobian of the inner function $\partial G_j(\cdot)$, $V_j$ for storing the value of the inner function $G_j(\cdot)$, and $Q_i$ for storing the gradient of the outer function $\nabla F_i(\cdot)$.
At iteration $k$, we randomly select indices $j_k$, $i_k$, and a mini-batch $A_k$. We update the $J_j$'s as follows:
\begin{equation}
	\label{evaluateJ}
	J_j^k = 
    \begin{aligned}
    	\begin{cases}
    		\partial G_j(x^{k-1}) &\textnormal{if } j = j_k\\
        	J_j^{k-1} &\textnormal{otherwise.}
    	\end{cases}
    \end{aligned}
\end{equation}
We update the $V_j$'s as follows:
\begin{equation}
	\label{evaluateV}
	V_j^k = 
    \begin{aligned}
    	\begin{cases}
    		G_j(x^{k-1}) &\textnormal{if } j \in A_k\\
            V_j^{k-1} &\textnormal{otherwise.}
    	\end{cases}
    \end{aligned}
\end{equation}
We update the $Q_i$'s as follows:
\begin{equation}
	\label{evaluateQ}
	Q_i^k = 
    \begin{aligned}
    	\begin{cases}
    		\nabla F_i\left( \frac{1}{m} \sum_{j=1}^m V_j^k \right) &\textnormal{if } i = i^k\\
            Q_i^{k-1} &\textnormal{otherwise.}
    	\end{cases}
    \end{aligned}
\end{equation}
Finally, the update rule is obtained by substituting Eq.(\ref{evaluateJ})-Eq.(\ref{evaluateQ}) into Eq.(\ref{chainrule}):
\begin{equation}
	\label{Equpdate}
	\begin{aligned}
		&x^{k} = x^{k-1} - \alpha \nabla \hat{f}^{k}\\
        &\nabla \hat{f}^{k} = (\partial \hat{G}(x^{k-1}))^T \nabla \hat{F}(\hat{G}(x^{k-1}))\\
        &\partial \hat{G}(x^{k-1}) = \frac{1}{m}\sum_{j=1}^m J_j^k\\
        &\hat{G}(x^{k-1}) = \frac{1}{m}\sum_{j=1}^m V_j^k\\
        &\nabla \hat{F}(\hat{G}(x^{k-1})) = \frac{1}{n} \sum_{i=1}^n Q_i^k\\
	\end{aligned}
\end{equation}

\begin{algorithm}[t]
  	\caption{C-SAG}
    \label{alg:c_sag}
  	\begin{algorithmic}[1]
      	\State \textbf{Require:} $K$ (update period), $\alpha$ (step size), $S$ (maximum number of training iterations), $a$ (mini batch size)
      	\While{$true$}
      		\State $\backslash\backslash$ Refresh memory, evaluate full gradient
        	\State $J_j = \partial G_j(\tilde{x})$ for $j = 1,...,m$ \Comment{m queries}
        	\State $V_j = G_j(\tilde{x})$ for $j = 1,...,m$ \Comment{m queries}
        	\State $Q_i = \nabla F_i(G(\tilde{x}))$ for $i = 1,...,n$ \Comment{n queries}
            \State $x_{0}=\tilde{x}-\alpha \nabla f(\tilde{x})$
        	\For{$k = 1$ to $K$}
        		\State Uniformly select $j_k$ from $\left[1,2,...,m\right]$, mini-batch $A_k$ from $\left[1,2,...,m\right]$ and $i_k$ from $\left[1,2,...,n\right]$
            	\State $\backslash\backslash$Estimate $\nabla \hat{f}^{k}$
                \State update memory $J$ by Eq.(\ref{evaluateJ}) \Comment{1 query}
                \State update memory $V$ by Eq.(\ref{evaluateV}) \Comment{$\vert A \vert$ queries}
                \State update memory $Q$ by Eq.(\ref{evaluateQ}) \Comment{1 query}
            	\State estimate $\partial \hat{G}(x^{k-1}) = \frac{1}{m} \sum_{j=1}^m J_j^k$
            	\State estimate $\hat{G}(x^{k-1}) = \frac{1}{m} \sum_{j=1}^m V_j^k$
            	\State estimate $\nabla \hat{F}(\hat{G}(x^{k-1})) = \frac{1}{n} \sum_{i=1}^n Q_i^k$
            	\State $\nabla \hat{f}^{k} = (\partial \hat{G}(x^{k-1}))^T \nabla \hat{F}(\hat{G}(x^{k-1}))$
            	\State $x^{k}=x^{k-1}-\alpha \nabla\hat{f}^{k}$
        	\EndFor
            \State $\tilde{x} = x^{K} $
       		\If{converged or $S$ reached}
            \State \Return $x^{K}$
        	\EndIf
      	\EndWhile
  	\end{algorithmic}
\end{algorithm}
In addition, because the  gradient estimate is biased in the case of FS-CEVF objective function \cite{lian2017finite}, to improve the stability of the algorithm, we use a refresh mechanism to evaluate the exact full gradient periodically and update $J$, $V$, and $Q$. The frequency of refresh is controlled by a pre-specified parameter $K$. The complete algorithm is shown in Algorithm \ref{alg:c_sag}.

\subsection{Oracle Query Complexity of C-SAG}

In the algorithm, we define the cost of an oracle query as the unit cost of computing $\partial G_j(\cdot)$, $G_j(\cdot)$, and $\nabla F_i(\cdot)$. Evaluating the exact full gradient, such as line 4, 5, and 6 in Algorithm \ref{alg:c_sag}, takes $m$, $m$, and $n$  oracle queries, for $\partial G(\cdot)$,  $G(\cdot)$, and $\nabla F(\cdot)$ respectively, yielding an oracle query complexity of $O(2m+n)$. On the other hand, lines 11 to 13 of  Algorithm \ref{alg:c_sag} estimate gradient by incorporating the memory of previous steps, yielding $2+\vert A_k \vert$ queries at each iteration, which is independent of the number of component functions.

We proceed to  compare the oracle query complexity per iteration between C-SAG with the state-of-the-art C-SVRG methods C-SVRG-1 and C-SVRG-2 \cite{lian2017finite}.  C-SVRG-1 requires $2A+4$ oracle queries in each iteration (where $A$ is the mini-batch size used by C-SVRG-1) except for the reference update iteration, in which exact full gradient is evaluated which requires $2m+n$ queries. The reference update iteration in C-SVRG methods is equivalent corresponds to the memory refreshing mechanism in C-SAG. In addition, C-SVRG-2 takes $2A+2B+2$ queries per iteration (where $B$ is a second mini-batch size used by C-SVRG-2) except for the reference update iteration. On the other hand, the proposed C-SAG takes $A+2$ queries per iteration except for the refresh iteration. The number of queries required by C-SAG during memory refresh iterations are the same as those for C-SVRG-1 and C-SVRG-2 during their reference update iterations. We conclude that C-SAG, in general, incurs only half the oracle query complexity of C-SVRG-1, and less than half that of C-SVRG-2.

\section{Convergence Analysis of C-SAG}
We proceed to establish the convergence and the convergence rate of C-SAG.  Because of  space constraints, the details of the proof are relegated to the  Appendix. 

Suppose the objective function in Eq.(\ref{objFunc}) has a minimizer $x^*$. In what follows, we make the same assumptions regarding the functions $f(\cdot)$, $F(\cdot)$, and $G(\cdot)$ as those used in \cite{lian2017finite} to establish the convergence of C-SVRG methods.  We assume:
\begin{itemize}
\item[(a)]
{\bf Strongly Convex Objective:}
$f(x)$ in Eq.(\ref{objFunc}) is strongly convex with parameter $\mu _f$:
\begin{equation*}
	\label{strongConvex}
    f(y) \geq f(x) + <\nabla f(x),y-x>+\frac{\mu _f}{2}\V y-x \V^2, \blank \forall x,y. \tag{$\Delta_1$}
\end{equation*}
\item[(b)]
{\bf Bounded Jacobian of Inner Functions}
We assume that the Jacobian of all inner component functions are upper-bounded by $B_G$,
\begin{equation*}
	\label{boundedIn}
    \V \partial G_j(x) \V \leq B_G, \blank \forall x, \forall j \in \left\lbrace 1,\ldots, m \right\rbrace. \tag{$\Delta_2$}
\end{equation*}
Using this inequality, we can bound the estimated Jacobian $J_j^{k}$. Because $J_j^{k} = \partial G_j^{k'}$ for some $k' < k$, we have:
\begin{equation*}
	\label{boundedJ}
    \V J_j \V \leq B_G, \blank \forall j \in \left\lbrace 1,\ldots, m \right\rbrace. \tag{$\Delta_3$}
\end{equation*}
\item[(c)]
{\bf Lipschitz Gradients:}
We assume that there is a constant $L_F$ such that $\forall x, \forall y, \forall i\in \left\lbrace 1, \ldots n \right\rbrace$, the following holds:
\begin{equation*}
	\label{lipschitzF}
    \V \nabla F_i(x) - \nabla F_i(y)\V \leq L_F\V x-y \V \tag{$\Delta_4$},
\end{equation*}
\end{itemize}
Armed with the preceding assumptions, we proceed to establish the convergence of C-SAG. 
\begin{theorem}\textnormal{(Convergence of C-SAG algorithm)}.
\textnormal{For the proposed C-SAG algorithm, we have}:\\
	\begin{equation*}
		\frac{1}{K}\sum_{k=0}^{K-1}\E \V x_k - x^* \V^2
        \leq \blank \frac{\gamma_1}{\gamma_2} \E \V \tilde{x}-x^* \V^2.
	\end{equation*}
where\\
	\begin{alignat*}{2}
		& \gamma_1 = \frac{1}{K}+\left( n\sigma_1 + 3m\left( 1-(\frac{a}{m})^2 \right)\sigma_2 \right)B_G^4 L_F^2\\
        & \gamma_2 = \alpha \mu_f - \left( \frac{32\alpha (m-a)}{m\mu_f} + 3a(2-\frac{a}{m}) \sigma_2 \right) B_G^4 L_F^2,
	\end{alignat*}
and\\
	\begin{equation*}
		\begin{alignedat}{2}
    		&\begin{alignedat}{2}
        		\sigma_1 = &9\alpha^2\left((m-1)^2(1-\frac{1}{n})(n + 2) + (n-1)(4m-3)\right)+16\alpha n \frac{(m-1)^2(1-\frac{1}{n})^2 + (1-\frac{1}{n})^2}{\mu_f} \\
        	\end{alignedat}\\
			&\begin{alignedat}{2}
        		\sigma_2 = &\frac{9\alpha^2}{m}\left( (m-1)^2(2n-1)+n+4n(m-1) \right)+16\alpha \frac{(1-\frac{1}{m})^2 m + 16(m-a)}{m^2\mu_f}.
        	\end{alignedat}
		\end{alignedat}
	\end{equation*}
\end{theorem}

\begin{proof}
The proof proceeds as follows: (i) We express $\E \V x^k - x^* \V^2$ in terms of $J^{k-1}$, $V^{k-1}$, $Q^{k-1}$, and $x^{k-1}$; (ii) We use the assumptions  (a) through (c) above  to bound $J^{k-1}$, $V^{k-1}$, and $Q^{k-1}$ in terms of $x^{k-1}$ and $\tilde{x}$; and (iii) We use the preceding bounds to establish the convergence of C-SAG, and the choice of parameters  (refresh frequency $1/K$, batch size $a$, and learning rate $\alpha$ etc.) that ensure convergence (See Appendix for details)
\end{proof}

\begin{corollary}\textnormal{(Convergence rate of C-SAG algorithm).
Suppose we choose the user-specified parameters of C-SAG such that:}
\begin{alignat*}{2}
	\begin{cases}
		& a > m \left( 1-\frac{\mu_f^2}{128 B_G^4 L_F^2} \right),\\
    	& \alpha < \min \left\lbrace \alpha_1,\alpha_2,\alpha_3 \right\rbrace,\\
    	& K > \frac{8}{\alpha \mu_f},
	\end{cases}
\end{alignat*}
\textnormal{where}
\begin{alignat*}{2}
	&\alpha_1 = \frac{m\left( \frac{\mu_f}{12a\left( 2-\frac{a}{m} \right) B_G^4 L_F^2} - 16\frac{(1-\frac{1}{m})^2 m + 16(m-a)}{m^2\mu_f}\right)}{9\left( (m-1)^2(2n-1)+n+4n(m-1) \right)}\\
    &\alpha_2 = \frac{\frac{\mu_f}{8n B_G^4 L_F^2} - 16 n \frac{(m-1)^2(1-\frac{1}{n})^2 + (1-\frac{1}{n})^2}{\mu_f}}{9\left((m-1)^2(1-\frac{1}{n})(n + 2) + (n-1)(4m-3)\right)} \\
    &\alpha_3 = \frac{m \left( \frac{\mu_f}{24m\left( 1-\left( \frac{a}{m} \right)^2 \right)B_G^4 L_F^2} - 16 \frac{(1-\frac{1}{m})^2 m + 16(m-a)}{m^2\mu_f} \right)}{9\left( (m-1)^2(2n-1)+n+4n(m-1) \right)},
\end{alignat*}
\textnormal{Then we have:}
\begin{alignat*}{2}
	&\frac{1}{K}\sum_{k=0}^{K-1}\E \V x_k - x^* \V^2 < \frac{3}{4} \E \V \tilde{x}-x^* \V^2.
\end{alignat*}
\end{corollary}
Corollary 1 implies a linear convergence rate \cite{lian2017finite}.

\section{Experimental evaluation of C-SAG}
We report results of comparison of C-SAG with two variants of C-SVRG \cite{lian2017finite}, C-SVRG-1 and C-SVRG-2, which are among the state-of-the-art methods for FS-CEVF optimization, as well as the classical full gradient (FG) method \cite{cauchy1847methode},  on the mean-variance portfolio optimization problem (See Section 2 for details).

Our experimental setup closely follows the setup in \cite{lian2017finite}. Synthetic data were drawn from  a multi-variate Gaussian distribution and absolute values of the data samples were retained. We generated two synthetic data sets (D1 and D2). D1 consists of 2000 time points, and 200 assets, i.e., $n=2000$, $N=200$; D2  consists of 5000 time points, and 300 assets, i.e., $n=5000$, $N=300$. We also controlled $\kappa_{cov}$, the condition number of the covariance matrix used in the multi-variate Gaussian distribution to generate the synthetic data. We compared our proposed algorithm to the classical FG method \cite{cauchy1847methode}, and two state-of-the-art algorithms, C-SVRG-1 \cite{lian2017finite} and C-SVRG-2 \cite{lian2017finite}. The initialization and the step sizes were chosen to be identical for all algorithms. The parameters are set to their default values in the implementation provided by the authors of \cite{lian2017finite}: The size of the mini-batch was set to 20 ($a=20$), the update period was set to 20 iterations ($K=20$), and the constant step size was set to 0.12 ($\alpha = 0.12$).

\begin{figure}[t]
	\begin{subfigure}{.5\textwidth}
    	\centering
    	\includegraphics[width=.8\linewidth]{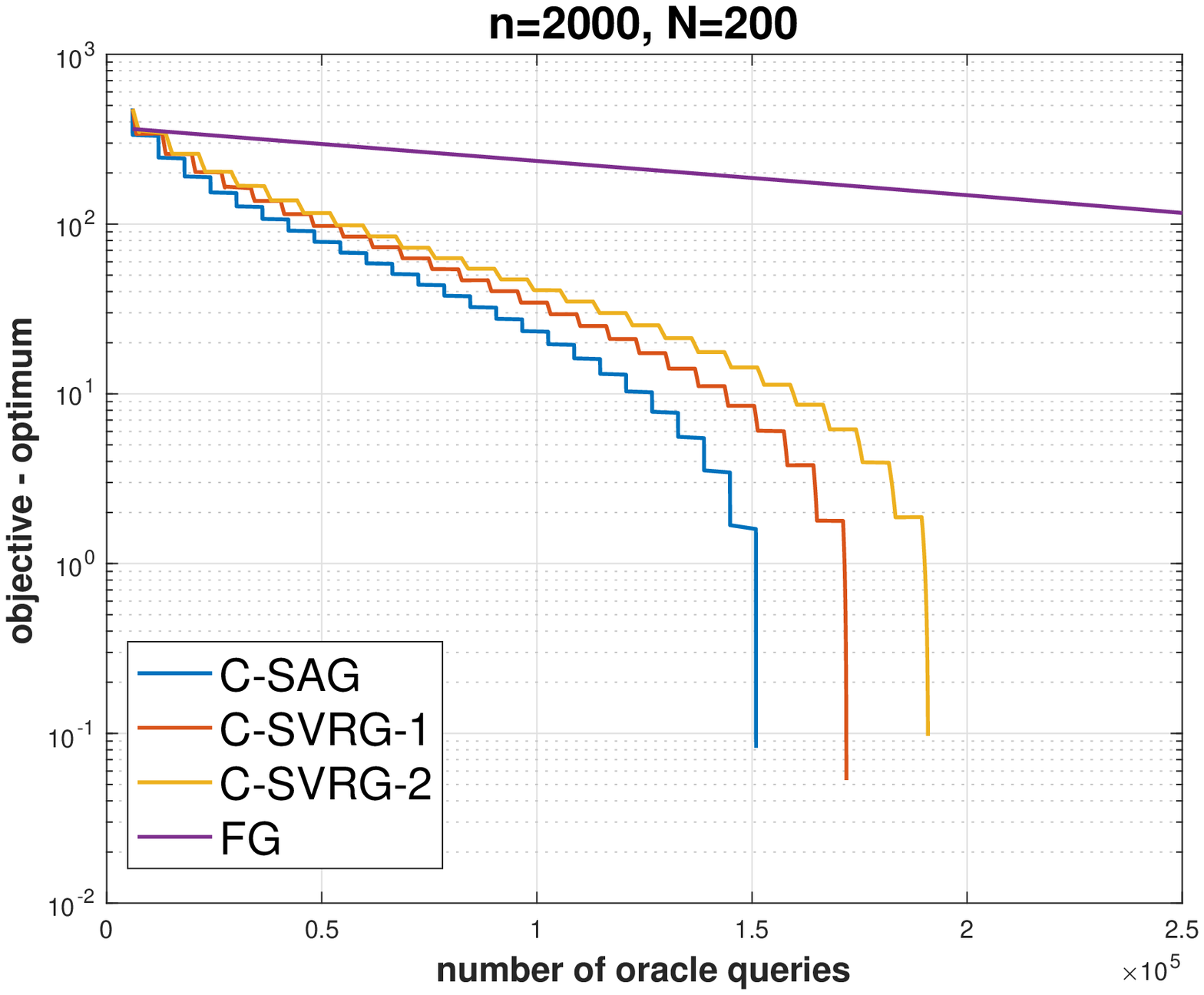}
    	\caption{$\kappa_{cov} = 20$}
    	\label{fig1:sfig1}
  	\end{subfigure}%
  	\begin{subfigure}{.5\textwidth}
    	\centering
    	\includegraphics[width=.8\linewidth]{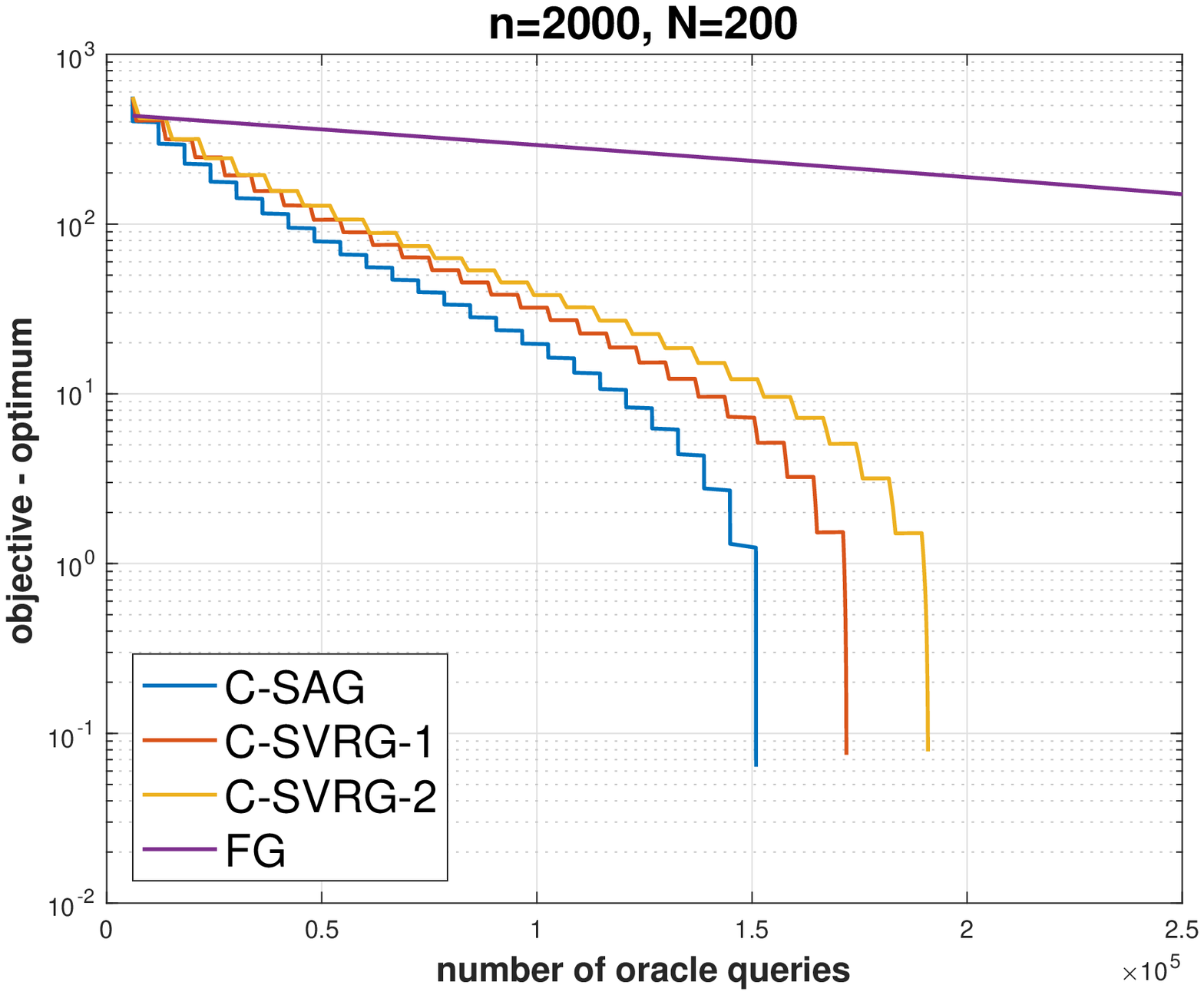}
    	\caption{$\kappa_{cov} = 100$}
    	\label{fig1:sfig2}
  	\end{subfigure}
	\caption{Mean-variance portfolio optimization on synthetic data D1 with 200 assets ($N=300$) and the number of time points $n=2000$. The logarithm of the objective value minus the optimum is plotted along the Y-axis and the number of oracle queries is plotted along the X-axis. The $\kappa_{cov}$ is the conditional number of the covariance matrix of the corresponding Gaussian distribution used to generate reward vectors.}
  	\label{fig1:fig}
\end{figure}
The experimental results comparing the different methods on data sets D1 and D2 are shown in Fig. 1 and Fig. 2 respectively.  The number of oracle queries is proportional to the runtime of the algorithm. Note that the sudden drop at the tail of each curve results from treating the smallest objective value in the convergence sequence, which is usually found at the end of the sequence, as the optimum value.
We observe that on both data sets, all stochastic gradient methods converge faster than full gradient (FG) method, and that  C-SAG achieves better (lower) value of the objective function  at each iteration. Although the three stochastic methods under comparison (C-SVRG-1, C-SVRG-2, and our method, C-SAG) all have linear convergence rates, C-SAG converges faster in practice by virtue of lower oracle query complexity per iteration as compared to C-SVRG-1 and C-SVRG-2.  

\begin{figure}
	\begin{subfigure}{.5\textwidth}
    	\centering
    	\includegraphics[width=.8\linewidth]{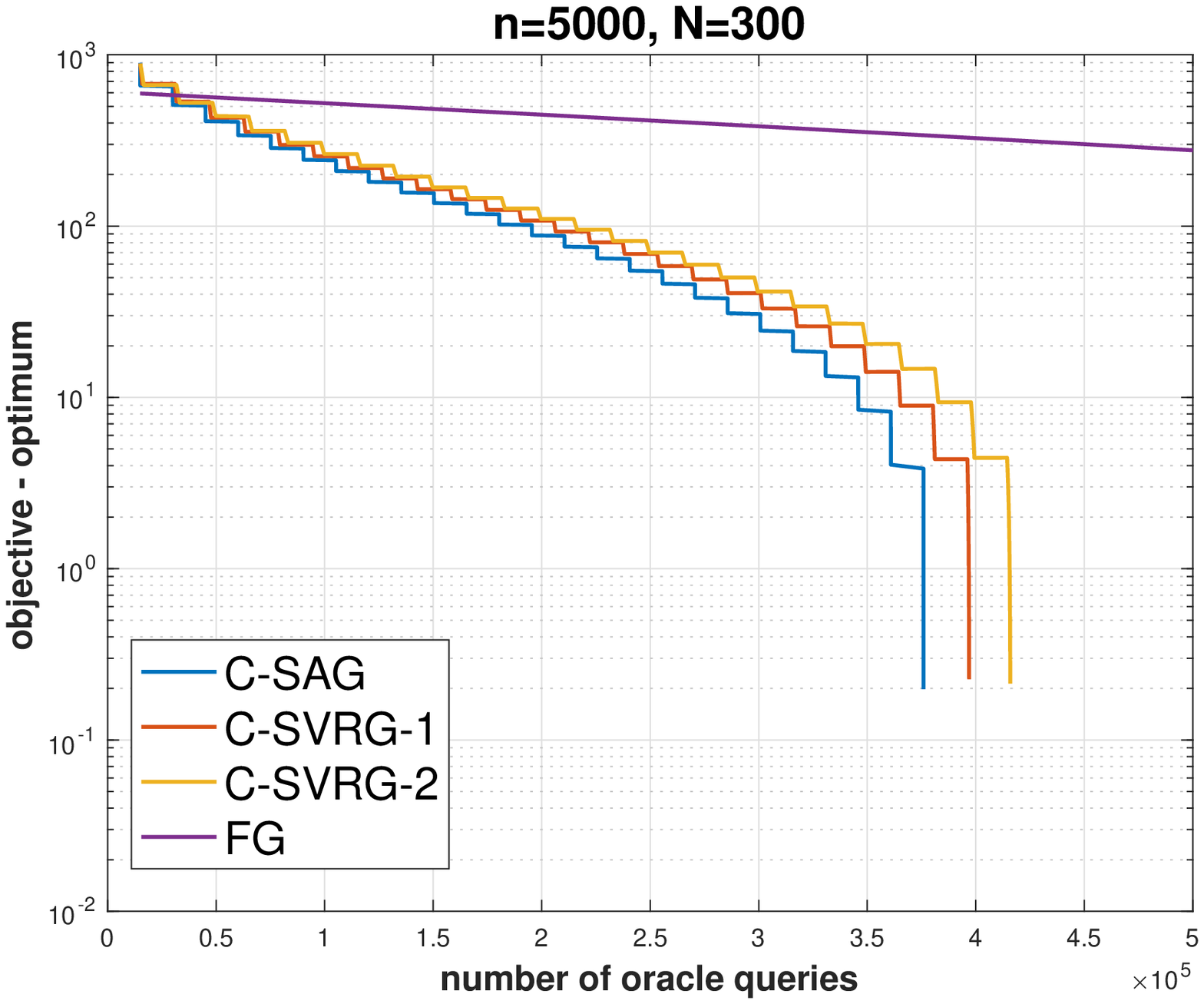}
    	\caption{$\kappa_{cov} = 20$}
    	\label{fig2:sfig1}
  	\end{subfigure}%
  	\begin{subfigure}{.5\textwidth}
    	\centering
    	\includegraphics[width=.8\linewidth]{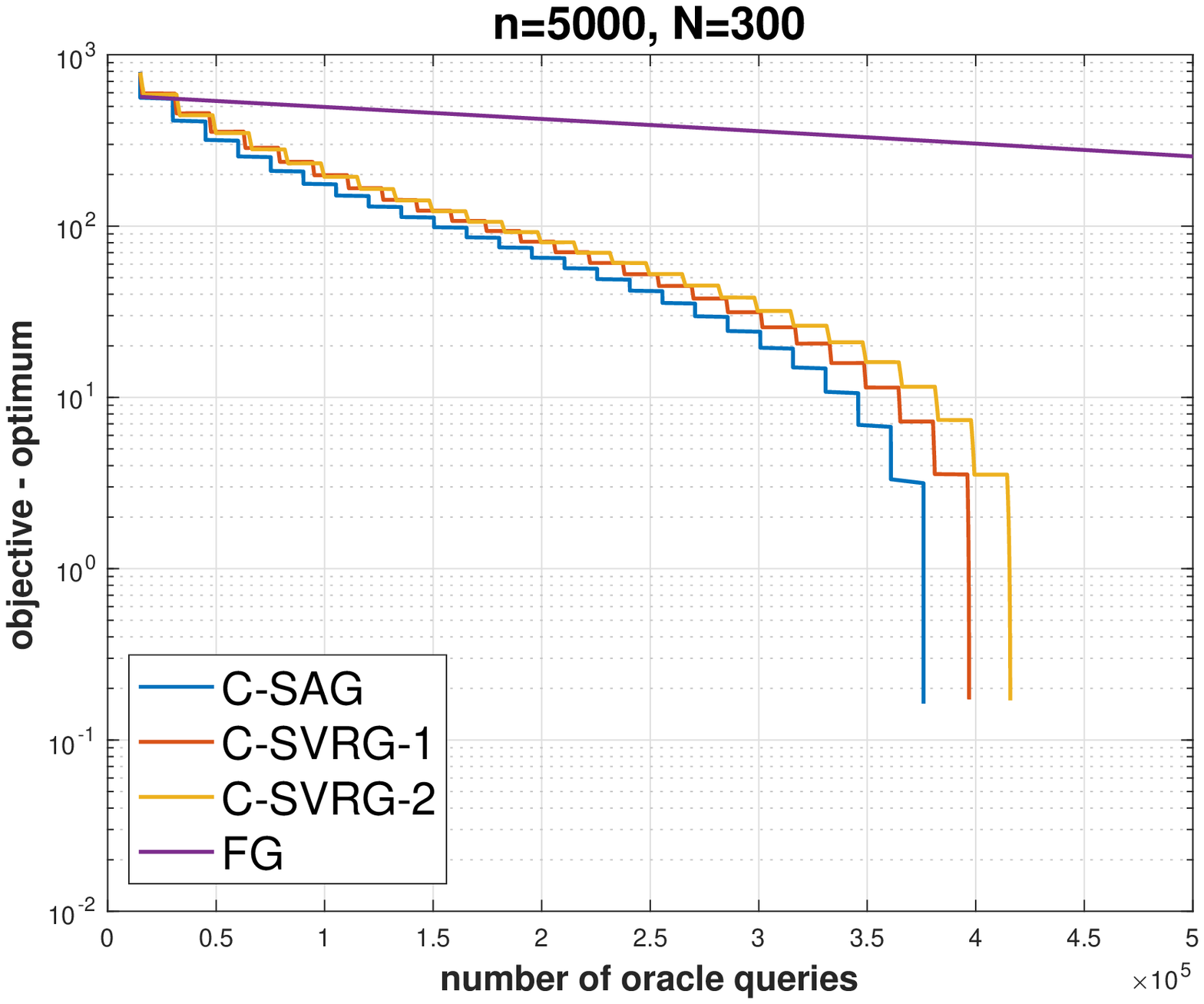}
    	\caption{$\kappa_{cov} = 100$}
    	\label{fig2:sfig2}
  	\end{subfigure}
	\caption{Mean-variance portfolio optimization on synthetic data D2 with 300 assets ($N=200$) and the number of time points $n=5000$. The logarithm of the objective value minus the optimum is plotted along the Y-axis and the number of oracle queries is plotted along the X-axis. The $\kappa_{cov}$ is the conditional number of the covariance matrix of the corresponding Gaussian distribution used to generate reward vectors.}
  	\label{fig2:fig}
\end{figure}

C-SAG has three user-specified parameters.  It is especially interesting as to how the update period, $K$, and the size of the mini-batch, $a$,  impact the convergence of C-SAG. The parameter $K$ controls the frequency at which  the full gradient is computed, and hence represents a trade-off between the accuracy and the computational efficiency of C-SAG. We experimented with C-SAG with $K$ set to 10, 20, 50, and 200 iterations. The results of this experiment are summarized in Fig. 3(a). As expected, C-SAG converges faster for larger values of $K$ (i.e., the time elapsed between consecutive memory refresh events or full gradient computation is longer). However, for larger values of $K$, C-SAG may fail to converge to the optimum value of the objective function. 
\begin{figure}[t]
	\begin{subfigure}{.5\textwidth}
    	\centering
    	\includegraphics[width=.8\linewidth]{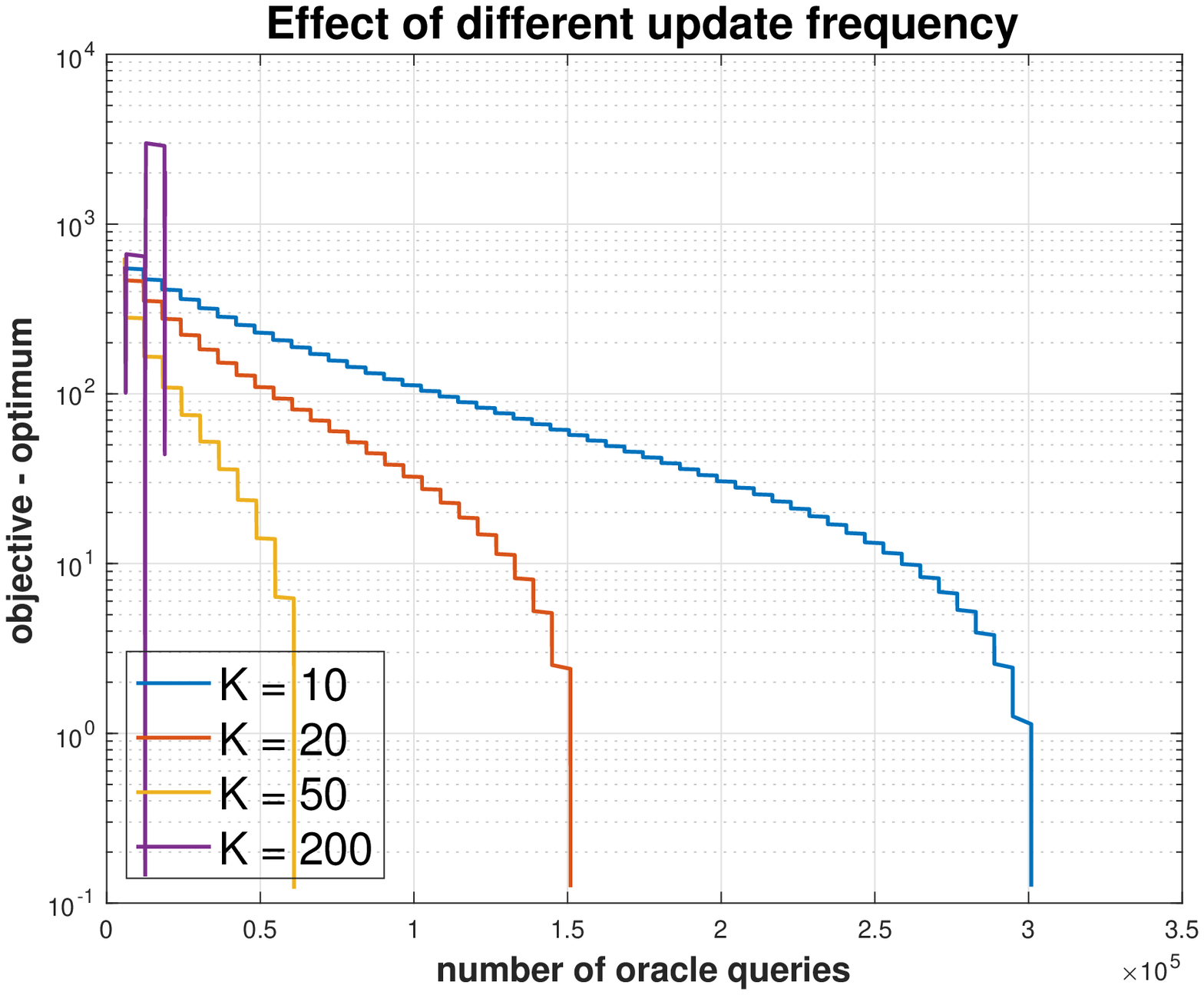}
    	\caption{Effect of different update period}
    	\label{fig3:sfig1}
  	\end{subfigure}%
  	\begin{subfigure}{.5\textwidth}
    	\centering
    	\includegraphics[width=.8\linewidth]{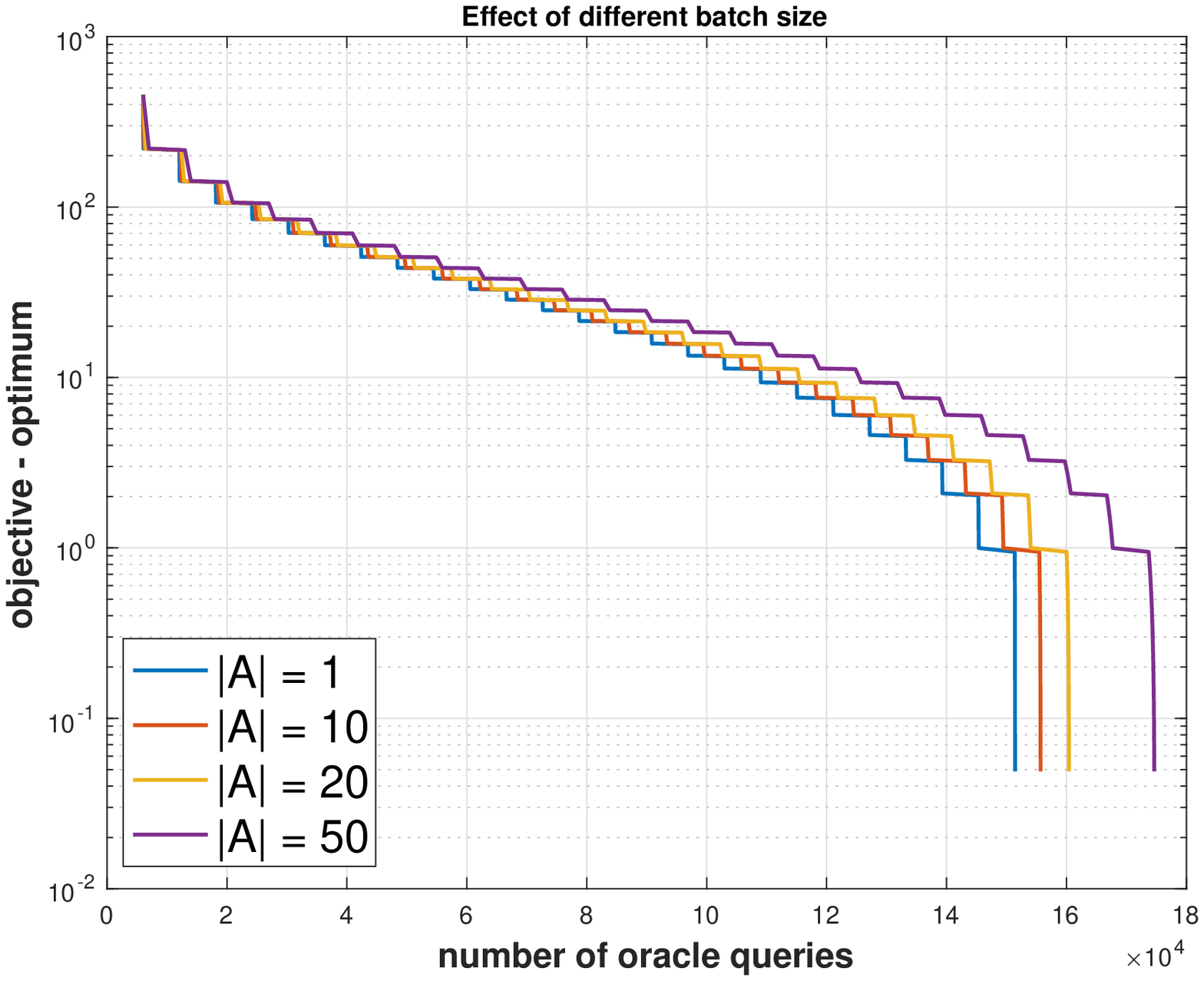}
    	\caption{Effect of different mini-batch size}
    	\label{fig3:sfig2}
  	\end{subfigure}
	\caption{Effect of user-specified parameters: (a) the duration between successive full gradient computations, $K$ and (b) mini-batch size, $\vert A \vert$, on the performance of C-SAG.}
  	\label{fig3:fig}
\end{figure}

A second parameter of interest is the size of the mini-batch used by C-SAG. We experiment with different mini-batch sizes, from 1, 10, 20, to 50. The  results are shown in Figure 3(b). Not surprisingly, we observe that C-SAG converges faster as  mini-batch size is decreased.

\section{Summary and Discussion}


Many machine learning, statistical inference, and portfolio optimization problems require minimization of a composition of expected value functions.   We  have introduced C-SAG, a novel extension of SAG, to minimize composition of finite-sum functions, i.e., finite sum variants of composition of expected value functions. We have established the convergence of the resulting algorithm and shown that it, like other state-of-the-art methods e.g., C-SVRG,  achieves a linear convergence rate when the objective function is strongly convex, while benefiting from lower oracle query complexity per iteration as compared to C-SVRG. We have presented results of experiments that show that C-SAG converges substantially faster than the state-of-the art C-SVRG variants.

Work in progress is aimed at (i) analyzing the convergence rate of C-SAG for general (weakly convex) problems; (ii) developing a distributed optimization algorithm for the composition finite-sum problems by combining C-SAG with the ADMM framework \cite{boyd2011distributed}; and (iii) applying C-SAG and its variants to problems of practical importance in machine learning. 
\section*{Acknowledgments}
This project was supported in part by the National Center for Advancing Translational Sciences, National Institutes of Health through the grant UL1 TR000127 and TR002014, by the National Science Foundation, through the grants 1518732, 1640834, and 1636795, the Pennsylvania State University’s Center for Big Data Analytics and Discovery  Informatics, the Edward Frymoyer Endowed Professorship in Information Sciences and Technology at Pennsylvania State University and the Sudha Murty Distinguished Visiting Chair in Neurocomputing and Data Science funded by the Pratiksha Trust at the Indian Institute of Science [both held by Vasant Honavar]. The content is solely the responsibility of the authors and does not necessarily represent the official views of the sponsors. 
%
%
%
\bibliographystyle{splncs04}

\bibliography{Bibliography.bib}

\begin{thebibliography}{10}
\providecommand{\url}[1]{\texttt{#1}}
\providecommand{\urlprefix}{URL }
\providecommand{\doi}[1]{https://doi.org/#1}

\bibitem{amari1993backpropagation}
Amari, S.I.: Backpropagation and stochastic gradient descent method.
  Neurocomputing  \textbf{5}(4-5),  185--196 (1993)

\bibitem{amini2008high}
Amini, A.A., Wainwright, M.J.: High-dimensional analysis of semidefinite
  relaxations for sparse principal components. In: Information Theory, IEEE
  International Symposium on. pp. 2454--2458. IEEE (2008)

\bibitem{baird1999gradient}
Baird~III, L.C., Moore, A.W.: Gradient descent for general reinforcement
  learning. In: Advances in neural information processing systems. pp. 968--974
  (1999)

\bibitem{bishop2006pattern}
Bishop, C.: Pattern recognition and machine learning: springer new york  (2006)

\bibitem{bottou1991stochastic}
Bottou, L.: Stochastic gradient learning in neural networks. Proceedings of
  Neuro-N{\i}mes  \textbf{91}(8), ~12 (1991)

\bibitem{bottou2010large}
Bottou, L.: Large-scale machine learning with stochastic gradient descent. In:
  Proceedings of COMPSTAT, pp. 177--186. Springer (2010)

\bibitem{bottou2018optimization}
Bottou, L., Curtis, F.E., Nocedal, J.: Optimization methods for large-scale
  machine learning. SIAM Review  \textbf{60}(2),  223--311 (2018)

\bibitem{boyd2011distributed}
Boyd, S., Parikh, N., Chu, E., Peleato, B., Eckstein, J., et~al.: Distributed
  optimization and statistical learning via the alternating direction method of
  multipliers. Foundations and Trends{\textregistered} in Machine learning
  \textbf{3}(1),  1--122 (2011)

\bibitem{cauchy1847methode}
Cauchy, A.: M{\'e}thode g{\'e}n{\'e}rale pour la r{\'e}solution des systemes
  d’{\'e}quations simultan{\'e}es. Comp. Rend. Sci. Paris  \textbf{25}(1847),
   536--538 (1847)

\bibitem{cauwenberghs1993fast}
Cauwenberghs, G.: A fast stochastic error-descent algorithm for supervised
  learning and optimization. In: Advances in neural information processing
  systems. pp. 244--251 (1993)

\bibitem{dai2016learning}
Dai, B., He, N., Pan, Y., Boots, B., Song, L.: Learning from conditional
  distributions via dual embeddings. arXiv preprint arXiv:1607.04579  (2016)

\bibitem{darken1990fast}
Darken, C., Moody, J.: Fast adaptive k-means clustering: some empirical
  results. In: Neural Networks, International Joint Conference on. pp.
  233--238. IEEE (1990)

\bibitem{dentcheva2017statistical}
Dentcheva, D., Penev, S., Ruszczy{\'n}ski, A.: Statistical estimation of
  composite risk functionals and risk optimization problems. Annals of the
  Institute of Statistical Mathematics  \textbf{69}(4),  737--760 (2017)

\bibitem{ermoliev1988stochastic}
Ermoliev, Y.: Stochastic quasigradient methods. Numerical techniques for
  stochastic optimization. No.~10, Springer (1988)

\bibitem{fagan2018unbiased}
Fagan, F., Iyengar, G.: Unbiased scalable softmax optimization. arXiv preprint
  arXiv:1803.08577  (2018)

\bibitem{friedman2001elements}
Friedman, J., Hastie, T., Tibshirani, R.: The elements of statistical learning,
  vol.~1. Springer series in statistics New York, NY, USA: (2001)

\bibitem{hu2014model}
Hu, J., Zhou, E., Fan, Q.: Model-based annealing random search with stochastic
  averaging. ACM Transactions on Modeling and Computer Simulation
  \textbf{24}(4), ~21 (2014)

\bibitem{huo2017accelerated}
Huo, Z., Gu, B., Huang, H.: Accelerated method for stochastic composition
  optimization with nonsmooth regularization. arXiv preprint arXiv:1711.03937
  (2017)

\bibitem{jain2018accelerating}
Jain, P., Kakade, S.M., Kidambi, R., Netrapalli, P., Sidford, A.: Accelerating
  stochastic gradient descent for least squares regression. In: Conference On
  Learning Theory. pp. 545--604 (2018)

\bibitem{jin2016provable}
Jin, C., Kakade, S.M., Netrapalli, P.: Provable efficient online matrix
  completion via non-convex stochastic gradient descent. In: Advances in Neural
  Information Processing Systems. pp. 4520--4528 (2016)

\bibitem{johnson2013accelerating}
Johnson, R., Zhang, T.: Accelerating stochastic gradient descent using
  predictive variance reduction. In: Advances in neural information processing
  systems. pp. 315--323 (2013)

\bibitem{kiefer1952stochastic}
Kiefer, J., Wolfowitz, J., et~al.: Stochastic estimation of the maximum of a
  regression function. The Annals of Mathematical Statistics  \textbf{23}(3),
  462--466 (1952)

\bibitem{kingma2014adam}
Kingma, D.P., Ba, J.: Adam: A method for stochastic optimization. arXiv
  preprint arXiv:1412.6980  (2014)

\bibitem{le2011optimization}
Le, Q.V., Ngiam, J., Coates, A., Lahiri, A., Prochnow, B., Ng, A.Y.: On
  optimization methods for deep learning. In: Proceedings of the 28th
  International Conference on International Conference on Machine Learning. pp.
  265--272. Omnipress (2011)

\bibitem{lecun2015deep}
LeCun, Y., Bengio, Y., Hinton, G.: Deep learning. nature  \textbf{521}(7553),
  ~436 (2015)

\bibitem{lecun2012efficient}
LeCun, Y.A., Bottou, L., Orr, G.B., M{\"u}ller, K.R.: Efficient backprop. In:
  Neural networks: Tricks of the trade, pp. 9--48. Springer (2012)

\bibitem{lian2017finite}
Lian, X., Wang, M., Liu, J.: Finite-sum composition optimization via variance
  reduced gradient descent. In: Artificial Intelligence and Statistics. pp.
  1159--1167 (2017)

\bibitem{lin2018improved}
Lin, T., Fan, C., Wang, M., Jordan, M.I.: Improved oracle complexity for
  stochastic compositional variance reduced gradient. arXiv preprint
  arXiv:1806.00458  (2018)

\bibitem{liu2017variance}
Liu, L., Liu, J., Tao, D.: Variance reduced methods for non-convex composition
  optimization. arXiv preprint arXiv:1711.04416  (2017)

\bibitem{mandt2017stochastic}
Mandt, S., Hoffman, M.D., Blei, D.M.: Stochastic gradient descent as
  approximate bayesian inference. The Journal of Machine Learning Research
  \textbf{18}(1),  4873--4907 (2017)

\bibitem{needell2014stochastic}
Needell, D., Ward, R., Srebro, N.: Stochastic gradient descent, weighted
  sampling, and the randomized kaczmarz algorithm. In: Advances in Neural
  Information Processing Systems. pp. 1017--1025 (2014)

\bibitem{rakhlin2011making}
Rakhlin, A., Shamir, O., Sridharan, K.: Making gradient descent optimal for
  strongly convex stochastic optimization. arXiv preprint arXiv:1109.5647
  (2011)

\bibitem{SpAM}
Ravikumar, P., Lafferty, J., Liu, H., Wasserman, L.: Sparse additive models.
  Journal of the Royal Statistical Society: Series B (Statistical Methodology)
  \textbf{71}(5),  1009--1030 (2009)

\bibitem{recht2011hogwild}
Recht, B., Re, C., Wright, S., Niu, F.: Hogwild: A lock-free approach to
  parallelizing stochastic gradient descent. In: Advances in neural information
  processing systems. pp. 693--701 (2011)

\bibitem{robbins1985stochastic}
Robbins, H., Monro, S.: A stochastic approximation method. In: Herbert Robbins
  Selected Papers, pp. 102--109. Springer (1985)

\bibitem{schmidt2012}
Roux, N.L., Schmidt, M., Bach, F.R.: A stochastic gradient method with an
  exponential convergence rate for finite training sets. In: Advances in neural
  information processing systems. pp. 2663--2671 (2012)

\bibitem{rumelhart1986learning}
Rumelhart, D.E., Hinton, G.E., Williams, R.J.: Learning representations by
  back-propagating errors. nature  \textbf{323}(6088), ~533 (1986)

\bibitem{schmidt2017minimizing}
Schmidt, M., Le~Roux, N., Bach, F.: Minimizing finite sums with the stochastic
  average gradient. Mathematical Programming  \textbf{162}(1-2),  83--112
  (2017)

\bibitem{shamir2016convergence}
Shamir, O.: Convergence of stochastic gradient descent for pca. In:
  International Conference on Machine Learning. pp. 257--265 (2016)

\bibitem{shapiro2009lectures}
Shapiro, A., Dentcheva, D., Ruszczy{\'n}ski, A.: Lectures on stochastic
  programming: modeling and theory. SIAM (2009)

\bibitem{sutton1998reinforcement}
Sutton, R.S., Barto, A.G., et~al.: Reinforcement learning: An introduction. MIT
  press (1998)

\bibitem{sutton2009fast}
Sutton, R.S., Maei, H.R., Precup, D., Bhatnagar, S., Silver, D.,
  Szepesv{\'a}ri, C., Wiewiora, E.: Fast gradient-descent methods for
  temporal-difference learning with linear function approximation. In:
  Proceedings of the 26th Annual International Conference on Machine Learning.
  pp. 993--1000. ACM (2009)

\bibitem{tan2016barzilai}
Tan, C., Ma, S., Dai, Y.H., Qian, Y.: Barzilai-borwein step size for stochastic
  gradient descent. In: Advances in Neural Information Processing Systems. pp.
  685--693 (2016)

\bibitem{theodoridis2015machine}
Theodoridis, S.: Machine learning: a Bayesian and optimization perspective.
  Academic Press (2015)

\bibitem{lasso}
Tibshirani, R.: Regression shrinkage and selection via the lasso. Journal of
  the Royal Statistical Society. Series B (Methodological) pp. 267--288 (1996)

\bibitem{wang2017accelerating}
Wang, L., Yang, Y., Min, R., Chakradhar, S.: Accelerating deep neural network
  training with inconsistent stochastic gradient descent. Neural Networks
  \textbf{93},  219--229 (2017)

\bibitem{wang2017stochastic}
Wang, M., Fang, E.X., Liu, H.: Stochastic compositional gradient descent:
  algorithms for minimizing compositions of expected-value functions.
  Mathematical Programming  \textbf{161}(1-2),  419--449 (2017)

\bibitem{wang2016accelerating}
Wang, M., Liu, J., Fang, E.: Accelerating stochastic composition optimization.
  In: Advances in Neural Information Processing Systems. pp. 1714--1722 (2016)

\bibitem{yufast}
Yu, Y., Huang, L.: Fast stochastic variance reduced admm for stochastic
  composition optimization. arXiv preprint arXiv:1705.04138  (2017)

\bibitem{yuan2007model}
Yuan, M., Lin, Y.: Model selection and estimation in the gaussian graphical
  model. Biometrika  \textbf{94}(1),  19--35 (2007)

\bibitem{zeiler2012adadelta}
Zeiler, M.D.: Adadelta: an adaptive learning rate method. arXiv preprint
  arXiv:1212.5701  (2012)

\bibitem{zhang2004solving}
Zhang, T.: Solving large scale linear prediction problems using stochastic
  gradient descent algorithms. In: Proceedings of the twenty-first
  international conference on Machine learning. p.~116. ACM (2004)

\bibitem{zhao2016fast}
Zhao, S.Y., Li, W.J.: Fast asynchronous parallel stochastic gradient descent: a
  lock-free approach with convergence guarantee. In: Proceedings of the
  Thirtieth AAAI Conference on Artificial Intelligence. pp. 2379--2385. AAAI
  Press (2016)

\bibitem{zinkevich2010parallelized}
Zinkevich, M., Weimer, M., Li, L., Smola, A.J.: Parallelized stochastic
  gradient descent. In: Advances in neural information processing systems. pp.
  2595--2603 (2010)

\end{thebibliography}

\newpage
\section*{Appendix}
We present here the convergence analysis of the C-SAG algorithm.

\subsection*{Preliminaries}
Let  $x^{*}$ be the unique minimizer of the $f$ in Eq.(\ref{objFunc}). Let the random variables $z_i^k$, $y_i^k$, and $w_i^k$ be defined as follows:
\begin{equation}
	\label{RVdef}
    \begin{aligned}
    	&z_i^k=
        \begin{cases}
        	1-\frac{1}{m}&\hspace{0.5cm},\textnormal{with probability} = \frac{1}{m}\\
            -\frac{1}{m}&\hspace{0.5cm},\textnormal{otherwise}
        \end{cases}\\
        &y_i^k=
        \begin{cases}
        	1-\frac{a}{m}&\hspace{0.5cm},\textnormal{with probability} = \frac{a}{m}\\
            -\frac{a}{m}&\hspace{0.5cm},\textnormal{otherwise}
        \end{cases}\\
        &w_i^k=
        \begin{cases}
        	1-\frac{1}{n}&\hspace{0.5cm},\textnormal{with probability} = \frac{1}{n}\\
            -\frac{1}{n}&\hspace{0.5cm},\textnormal{otherwise}
        \end{cases},
    \end{aligned}
\end{equation}
where $a$ is the size of the mini batch $A_k$. 
We can rewrite $J_j$, $V_j$, and $Q_i$ as
\begin{equation}
	\label{rewrite}
    \begin{aligned}
    	& J_j^k = (1-\frac{1}{m})J_j^{k-1} + \frac{1}{m} \partial G_j^{k-1} + z_j^k\left[ \partial G_j^{k-1}-J_j^{k-1} \right],\\
		& V_j^k = (1-\frac{a}{m})V_j^{k-1} + \frac{a}{m} G_j^{k-1} + y_j^k\left[ G_j^{k-1}-V_j^{k-1} \right],\\
        & Q_j^k = (1-\frac{1}{n})Q_i^{k-1} + \frac{1}{n} \nabla F_i(\hat{G}^{k-1}) + w_i^k\left[ \nabla F_i(\hat{G}^{k-1})-Q_i^{k-1} \right],\\
    \end{aligned}
\end{equation}
and therefore Eq.(\ref{Equpdate}) is equivalent to
\begin{equation}
	\label{EqUpdateRewrite}
	\begin{alignedat}{3}	
    		x^{k}&=x^{k-1} - &&\alpha \nabla \hat{f}(x^{k-1})\\
        	&=x^{k-1} - &&\alpha \left( \partial \hat{G}^{k-1} \right)^T \left( \nabla \hat{F}(\hat{G}^{k-1} \right)\\
        	&=x^{k-1} - &&\alpha \left( \frac{1}{m}\sum_{j=1}^m J_j^k \right)^T \left( \frac{1}{n}\sum_{i=1}^n Q_i^k \right)\\
        	&=x^{k} - &&\frac{\alpha}{mn}\left( (1-\frac{1}{m})e^T J^{k-1} + \partial G^{k-1} + (z^k)^T\left[ G'^{k-1}-J^{k-1} \right] \right)^T\\
        	& &&\left( (1-\frac{1}{n})e^T Q^{k-1} + \nabla F(\hat{G}^{k-1}) + (w^k)^T\left[ F'(\hat{G}^{k-1})-Q^{k-1} \right] \right)\\
            & = x^{k-1} - &&\frac{\alpha}{mn} S
	\end{alignedat}
\end{equation}
with
\begin{equation*}
	\begin{alignedat}{3}
		&e = \left(
    	\begin{tabular}{c}
    		$I$\\
        	$\vdots$\\
        	$I$
    	\end{tabular}\right) \in \R^{mq\times q},
        \blank &&G'^{k-1}= \left(
        \begin{tabular}{c}
    		$\partial G_1^{k-1}$\\
        	$\vdots$\\
        	$\partial G_m^{k-1}$
    	\end{tabular}\right) \in \R^{mq},
        \blank &&F'^{k-1}= \left(
        \begin{tabular}{c}
    		$\nabla F_1^{k-1}$\\
        	$\vdots$\\
        	$\nabla F_n^{k-1}$
    	\end{tabular}\right) \in \R^{mq\times p},\\
        &J^{k-1} = \left(
    	\begin{tabular}{c}
    		$J_1^{k-1}$\\
        	$\vdots$\\
        	$J_m^{k-1}$
    	\end{tabular}\right) \in \R^{mq\times p},
        \blank &&Q^{k-1} = \left(
    	\begin{tabular}{c}
    		$Q_1^{k-1}$\\
        	$\vdots$\\
        	$Q_m^{k-1}$
    	\end{tabular}\right) \in \R^{mq},
        \blank &&z^{k} = \left(
    	\begin{tabular}{c}
    		$z_1^{k}I$\\
        	$\vdots$\\
        	$z_m^{k}I$
    	\end{tabular}\right) \in \R^{mq\times q},\\
        &w^{k} = \left(
    	\begin{tabular}{c}
    		$w_1^{k}I$\\
        	$\vdots$\\
        	$w_m^{k}I$
    	\end{tabular}\right) \in \R^{mq\times q}.
	\end{alignedat}
\end{equation*}
In the last line of Eq.(\ref{EqUpdateRewrite}), we denote
\begin{equation*}
	\begin{alignedat}{2}
		S = &\left( (1-\frac{1}{m})e^T J^{k-1} + \partial G^{k-1} + (z^k)^T\left[ G'^{k-1}-J^{k-1} \right] \right)^T\\
        	&\left( (1-\frac{1}{n})e^T Q^{k-1} + \nabla F(\hat{G}^{k-1}) + (w^k)^T\left[ F'(\hat{G}^{k-1})-Q^{k-1} \right] \right).
	\end{alignedat}
\end{equation*}
Expanding the matrix multiplication, we obtain
\begin{equation}
	\label{S}
    \begin{aligned}
    	S
    	=&(1-\frac{1}{m})(1-\frac{1}{n})(J^{k-1})^Te e^TQ^{k-1}\\
        +&(1-\frac{1}{m})(J^{k-1})^Te \nabla F(\hat{G}^{k-1})\\
        +&(1-\frac{1}{m})(J^{k-1})^Te (w^k)^T\left[ F'(\hat{G}^{k-1})-Q^{k-1} \right]\\
        +&(1-\frac{1}{n})(\partial G^{k-1})^T e^T Q^{k-1}\\
        +&(\partial G^{k-1})^T \nabla F(\hat{G}^{k-1})\\
        +&(\partial G^{k-1})^T (w^k)^T\left[ F'(\hat{G}^{k-1})-Q^{k-1} \right]\\
        +&(1-\frac{1}{n})\left[ G'^{k-1}-J^{k-1} \right]^T(z^k) e^T Q^{k-1}\\
        +&\left[ G'^{k-1}-J^{k-1} \right]^T(z^k) \nabla F(\hat{G}^{k-1})\\
        +&\left[ G'^{k-1}-J^{k-1} \right]^T(z^k) (w^k)^T\left[ F'(\hat{G}^{k-1})-Q^{k-1} \right].
    \end{aligned}
\end{equation}

In addition to the assumptions in Section 4, we make use of several inequalities in our proof. First, for any $\alpha > 0$, we have: 
\begin{equation}
	\label{inequ1}
    \frac{1}{\alpha} \V x \V^2 + \alpha \V y \V ^2 \geq \vert \left\langle x,y \right\rangle \vert \tag{$\Box_1$}.
\end{equation}
Second, for any $x$ we have:
\begin{equation}
	\label{inequ2}
    \V x_1 + x_2 + \ldots + x_t \V^2 \leq t\left( \V x_1 \V^2 + \ldots + \V x_t \V^2 \right), \blank \forall t\in\mathbb{N}^+. \tag{$\Box_2$}
\end{equation}
Third, for any matrix $A$ and vector $x$, we have:
\begin{equation}
	\label{inequ3}
    \V Ax \V^2 \leq \V A \V^2 \V x \V^2 \tag{$\Box_3$}
\end{equation}
Lastly, note that $\nabla f(x^*) = (\partial G(x^*))^T\nabla F(G(x^*)) = 0$.
\subsection*{Proof of Convergence of C-SAG}
We decompose the expectation $\E \Vert x_k - x^* \Vert^2$ as follows:
\begin{equation}
	\label{decomp}
	\begin{alignedat}{2}
		\E \V x^k - x^* \V^2
        & = \E \V x^{k-1} - x^* \V^2 + \E \V x^k - x^{k-1} \V ^2 + 2\E \left\langle x^k-x^{k-1},x^{k-1}-x^*\right\rangle\\
        & = \E \V x^{k-1} -x^* \V^2 + \alpha^2 \E \V S \V^2 - 2\alpha \E \left\langle S,x^{k-1},x^*\right\rangle.
\end{alignedat}
\end{equation}
Recalling  that $\E[z^k] = 0$, $\E[w^k] = 0$, and $z^k$, $w^k$ are independent, we can rewrite the last term as:
\begin{alignat*}{3}
	\E \left\langle S,x^{k-1},x^* \right\rangle&=&&(1-\frac{1}{m})(1-\frac{1}{n})\left\langle(J^{k-1})^Te e^TQ^{k-1},x^{k-1}-x^*\right\rangle \tag{$s_1$}\\
	&+&&(1-\frac{1}{m})\left\langle(J^{k-1})^Te \nabla F(\hat{G}^{k-1}),x^{k-1}-x^*\right\rangle \tag{$s_2$}\\
	&+&&(1-\frac{1}{n})\left\langle(\partial G^{k-1})^T e^T Q^{k-1},x^{k-1}-x^*\right\rangle \tag{$s_3$}\\
	&+&&\left\langle(\partial G^{k-1})^T \nabla F(\hat{G}^{k-1}),x^{k-1}-x^*\right\rangle \tag{$s_4$}.
\end{alignat*}
Letting $\beta_1 = (1-\frac{1}{m})(1-\frac{1}{n})$ we have:  
\begin{equation}
	\label{s1}
    \begin{alignedat}{2}
    	s_1
        &= (1-\frac{1}{m})(1-\frac{1}{n})\left\langle(J^{k-1})^Te e^TQ^{k-1},x^{k-1}-x^*\right\rangle\\
        &\overset{\Box_1}{\geq} -\frac{\beta_1}{\alpha_1}\V (J^{k-1})^Tee^TQ^{k-1} \V^2 - \alpha_1 \beta_1 \V x^{k-1}-x^* \V^2\\
        &\overset{\Box_3}{\geq} -\frac{\beta_1}{\alpha_1} \V e^T J^{k-1} \V^2 \V e^TQ^{k-1} \V^2 - \alpha_1 \beta_1 \V x^{k-1} - x^* \V^2 \\
        &= -\frac{\beta_1}{\alpha_1} \V \sum_{j=1}^m J_j^{k-1} \V^2 \V \sum_{i=1}^n Q_i^{k-1} - \nabla F_i(G^*) \V^2 - \alpha_1 \beta_1 \V x^{k-1} - x^* \V^2\\
        &\overset{\overset{\Box_2}{\Delta_3}}{\geq} -\frac{\beta_1}{\alpha_1}m^2B_G^2 n\underset{T_0}{\underbrace{\sum_{i=1}^n\V Q_i^{k-1} - \nabla F_i(G^*) \V^2}}- \alpha_1 \beta_1 \V x^{k-1} - x^* \V^2,
    \end{alignedat}
\end{equation}
for any $\alpha_1>0$. Letting $\beta_2=1-\frac{1}{m}$, we have:
\begin{equation}
	\label{s2}
    \begin{alignedat}{2}
    	s_2
        &= (1-\frac{1}{m})\left\langle(J^{k-1})^Te \nabla F(\hat{G}^{k-1}),x^{k-1}-x^*\right\rangle\\
        &\overset{\Box_1}{\geq} -\frac{\beta_2}{\alpha_2}\V (J^{k-1})^Te\nabla F(\hat{G}^{k-1}) \V^2 - \alpha_2 \beta_2 \V x^{k-1}-x^* \V^2\\
        &\overset{\Box3}{\geq} -\frac{\beta_2}{\alpha_2}\V e^TJ^{k-1} \V^2 \V \nabla F(\hat{G}^{k-1}) - \nabla F(G^*) \V^2 - \alpha_2 \beta_2 \V x^{k-1}-x^* \V^2\\
        &\overset{\Delta_4}{\geq} -\frac{\beta_2}{\alpha_2}B_G^2L_F^2 \V \hat{G}^{k-1} - G^* \V^2 - \alpha_2 \beta_2 \V x^{k-1}-x^* \V^2\\
        &\geq -\frac{\beta_2}{\alpha_2} \frac{B_G^2L_F^2}{m}\underset{T_1}{\underbrace{ \sum_{j=1}^m \V V_j^k - G_j^* \V^2}} - \alpha_2 \beta_2 \V x^{k-1}-x^* \V^2
    \end{alignedat}
\end{equation}
for any $\alpha_2>0$. Similarly letting $\beta_3 = 1-\frac{1}{n}$, we have:
\begin{equation}
	\label{s3}
    \begin{alignedat}{2}
    	s_3
        &=(1-\frac{1}{n})\left\langle(\partial G^{k-1})^T e^T Q^{k-1},x^{k-1}-x^*\right\rangle \\
        &\overset{\Box_1}{\geq} -\frac{\beta_3}{\alpha_3}\V (\partial G^{k-1})^T e^T Q^{k-1}) \V^2 - \alpha_3 \beta_3 \V x^{k-1}-x^* \V^2\\
        &\overset{\Box_3}{\geq} -\frac{\beta_3}{\alpha_3} \V \partial G^{k-1}) \V^2 \V e^TQ^{k-1} \V^2 - \alpha_3 \beta_3 \V x^{k-1}-x^* \V^2\\
        &= -\frac{\beta_3}{\alpha_3} \V \partial G^{k-1}) \V^2 \V \sum_{i=1}^n Q_i^{k-1} - \nabla F_i(G^*) \V^2 - \alpha_3 \beta_3 \V x^{k-1}-x^* \V^2\\
        &\overset{\Delta_2}{\geq} -\frac{\beta_3}{\alpha_3}B_G^2nT_0 - \alpha_3 \beta_3 \V x^{k-1}-x^* \V^2,
    \end{alignedat}
\end{equation}
for any $\alpha_3>0$. Lastly, we have:
\begin{equation}
	\label{s4_short}
    \begin{alignedat}{2}
    	s_4
        &=\left\langle(\partial G^{k-1})^T \nabla F(\hat{G}^{k-1}),x^{k-1}-x^*\right\rangle\\
        &=\left\langle(\partial G^{k-1})^T \nabla F(\hat{G}^{k-1}) - \nabla f(x^{k-1}),x^{k-1}-x^*\right\rangle + \left\langle \nabla f(x^{k-1})),x^{k-1}-x^*\right\rangle\\
        &\overset{\Delta_1}{\geq} \underset{t_0}{\underbrace{\left\langle(\partial G^{k-1})^T \nabla F(\hat{G}^{k-1}) - \nabla f(x^{k-1}),x^{k-1}-x^*\right\rangle}} + \mu_f \V x^{k-1} - x^*\V^2,
    \end{alignedat}
\end{equation}
and
\begin{equation}
	\label{t0}
    \begin{alignedat}{2}
    	t_0
        &\overset{\Box_1}{\geq} -\frac{8}{\mu_f}\V (\partial G^{k-1})^T \nabla F(\hat{G}^{k-1}) - \nabla f(x^{k-1}) \V^2 - \frac{\mu_f}{8} \V x^{k-1}-x^* \V^2\\
        &= -\frac{8}{\mu_f}\V (\partial G^{k-1})^T \nabla F(\hat{G}^{k-1}) - (\partial G^{k-1})^T \nabla F(G^{k-1}) \V^2 - \frac{\mu_f}{8} \V x^{k-1}-x^* \V^2\\
        &\overset{\Box_3}{\geq} -\frac{8}{\mu_f}\V \partial G^{k-1} \V^2 \V \nabla F(\hat{G}^{k-1}) - \nabla F(G^{k-1}) \V^2 - \frac{\mu_f}{8} \V x^{k-1}-x^* \V^2\\
        &\overset{\overset{\Delta_2}{\Delta_5}}{\geq} -\frac{8}{\mu_f}B_G^2 L_F^2 \V \hat{G}^{k-1} - G^{k-1} \V^2 - \frac{\mu_f}{8} \V x^{k-1}-x^* \V^2.\\
    \end{alignedat}
\end{equation}
Since there are $a$ terms among $V_j^{k}, j\in \left[ 1,\ldots,m \right]$, which get evaluated, we can write: 
\begin{alignat*}{2}
	\E \V \hat{G}^{k-1} - G^{k-1} \V^2
	&=\E \frac{1}{m^2}\V \sum_{j \notin A_k}^m \left(V_j^{k} - G_j^{k-1} \right) + \sum_{j \in A_k}^m \left( G_j^{k-1} - G_j^{k-1} \right) \V^2\\
	&\leq \frac{m-a}{m^2} \sum_{j=1}^m \V V_j^{k} - G_j^{k-1} \V^2\\
	&= \frac{m-a}{m^2} \sum_{j=1}^m \V V_j^{k} - G_j^* + G_j^* - G_j^{k-1} \V^2\\
	&\leq \frac{2(m-a)}{m^2} \sum_{j=1}^m \V (V_j^{k} - G_j^*)\V^2 + \V (G_j^* - G_j^{k-1}) \V^2\\
    &\leq \frac{2(m-a)}{m^2} \left(T_1 + m B_G^2 \V x^{k-1} - x^* \V^2 \right),
\end{alignat*}
Hence, $s_4$ can be bounded by:
\begin{equation}
	\label{s4}
    \begin{alignedat}{2}
    	s_4
        &\geq -\frac{8}{\mu_f}B_G^2 L_F^2 \V \hat{G}^{k-1} - G^{k-1} \V^2 + \frac{7}{8}\mu_f \V x^{k-1}-x^* \V^2\\
        &\geq -\frac{16(m-a)}{m^2\mu_f}B_G^2 L_F^2 T_1 + \left( \frac{7}{8}\mu_f - \frac{16(m-a)}{m\mu_f}B_G^4 L_F^2\right) \V x^{k-1}-x^* \V^2.
    \end{alignedat}
\end{equation}
Adding up $s_1$ through $s_4$, we have:
\begin{equation}
	\label{boundInnerProd}
	\begin{alignedat}{2}
		\E \left\langle S,x^{k-1},x^* \right\rangle
        \geq &\left( -\frac{\beta_1}{\alpha_1}m^2 -\frac{\beta_3}{\alpha_3} \right)nB_G^2 T_0\\
        +&\left( -\frac{\beta_2}{m\alpha_2} -\frac{16(m-a)}{m^2\mu_f} \right)B_G^2L_F^2 T_1\\
        +&\left( \frac{7}{8}\mu_f - \frac{16(m-a)}{m\mu_f}B_G^4 L_F^2 -\alpha_1 \beta_1 -\alpha_2 \beta_2 -\alpha_3 \beta_3 \right)\V x^{k-1} -x^* \V^2.
	\end{alignedat}
\end{equation}

Next, we shall derive the inequality for the $\E \V S \V^2$ term. Using the inequality ($\Box_2$) we obtain:
\begin{alignat}{3}
	\E \V S \V^2 &\leq9 &&\left((1-\frac{1}{m})^2 (1-\frac{1}{n})^2 \V (J^{k-1})^T e e^T Q^{k-1} \V^2 \right. \tag{$s_5$}\\
    & &&+(1-\frac{1}{m})^2 \V (J^{k-1})^T e \nabla F(\hat{G}^{k-1}) \V^2 \tag{$s_6$}\\
    & &&+(1-\frac{1}{m})^2 \V (J^{k-1})^T e (w^k)^T\left[ F'(\hat{G}^{k-1}) - Q^{k-1} \right] \V^2 \tag{$s_7$}\\
    & &&+(1-\frac{1}{n})^2 \V (\partial G^{k-1})^T e^T Q^{k-1}  \V^2 \tag{$s_8$}\\
    & &&+\V (\partial G^{k-1})^T \nabla F(\hat{G}^{k-1}) \V^2 \tag{$s_9$}\\
    & &&+\V (\partial G^{k-1})^T (w^k)^T\left[ F'(\hat{G}^{k-1}) - Q^{k-1} \right] \V^2 \tag{$s_{10}$}\\
    & &&+(1-\frac{1}{n})^2 \V \left[G'^{k-1} - J^{k-1}\right]^T (z^k) e^T Q^{k-1} \V^2 \tag{$s_{11}$}\\
    & &&+\V \left[G'^{k-1} - J^{k-1}\right]^T (z^k) \nabla F(\hat{G}^{k-1}) \V^2 \tag{$s_{12}$}\\
    & &&+\left.\V \left[G'^{k-1} - J^{k-1}\right]^T (z^k) (w^k)^T\left[ F'(\hat{G}^{k-1}) - Q^{k-1} \right] \V^2 \tag{$s_{13}$} \right).
\end{alignat}
The  term $s_5$ can be bounded by:
\begin{equation}
	\label{s5}
    \begin{alignedat}{2}
    	s_5
        &=(1-\frac{1}{m})^2 (1-\frac{1}{n})^2 \V (J^{k-1})^T e e^T Q^{k-1} \V^2\\
        &\leq (1-\frac{1}{m})^2 (1-\frac{1}{n})^2 \V e^T (J^{k-1}) \V^2 \V e^T Q^{k-1} \V^2 \\
        &\leq (1-\frac{1}{m})^2 (1-\frac{1}{n})^2  m^2 B_G^2  \V \sum_{i=1}^n Q_i^{k-1} - \nabla F_i(G^*) \V^2\\
        &\leq (m-1)^2 (1-\frac{1}{n})^2 B_G^2  n T_0.
    \end{alignedat}
\end{equation}
The term $s_6$ can be bounded by:
\begin{equation}
	\label{s6}
    \begin{alignedat}{2}
    	s_6
        &= (1-\frac{1}{m})^2 \V (J^{k-1})^T e \nabla F(\hat{G}^{k-1}) \V^2\\
        &\leq (1-\frac{1}{m})^2 \V e^T (J^{k-1}) \V^2 \V \nabla F(\hat{G}^{k-1}) - \nabla F(G^*) \V^2\\
        &\leq (m - 1)^2 B_G^2 L_F^2 \frac{T_1}{m}.
    \end{alignedat}
\end{equation}
To bound the term $s_7$, we observe that:
\begin{equation}
	\label{s7}
    \begin{alignedat}{2}
    	s_7
        &= (1-\frac{1}{m})^2 \V (J^{k-1})^T e (w^k)^T\left[ F'(\hat{G}^{k-1}) - Q^{k-1} \right] \V^2\\
        &\leq (1-\frac{1}{m})^2 \V e^T (J^{k-1}) \V^2 \V (w^k) \V^2 \V F'(\hat{G}^{k-1}) - Q^{k-1} \V^2\\
        &\leq (1-\frac{1}{m})^2 m^2 B_G^2 \V (w^k) \V^2 \V F'(\hat{G}^{k-1}) - F'(G^*) + F'(G^*) - Q^{k-1} \V^2\\
        &\leq 2(1-\frac{1}{m})^2 m^2 B_G^2 \V (w^k) \V^2 \left( \V F'(\hat{G}^{k-1}) - F'(G^*) \V^2  + \V F'(G^*) - Q^{k-1} \V^2 \right)\\
        &\leq 2(1-\frac{1}{m})^2 m^2 B_G^2 \V (w^k) \V^2 \left( n L_F^2\V \hat{G}^{k-1} - G^* \V^2  + \V Q^{k-1} - F'(G^*) \V^2 \right)\\
        &\leq 2(m - 1)^2 B_G^2 \V (w^k) \V^2 \left( n L_F^2 \frac{T_1}{m}  + T_0 \right).
    \end{alignedat}
\end{equation}
Next, the term $s_8$ can be bounded by:
\begin{equation}
	\label{s8}
    \begin{alignedat}{2}
    	s_8
        &=  (1-\frac{1}{n})^2 \V (\partial G^{k-1})^T e^T Q^{k-1}  \V^2\\
        &\leq (1-\frac{1}{n})^2 \V \partial G^{k-1} \V^2 \V \sum_{i=1}^n Q_i^{k-1} - \nabla F_i(G^*) \V^2\\
        &\leq (1-\frac{1}{n})^2 B_G^2 n T_0.
    \end{alignedat}
\end{equation}
The term $s_9$ can be bounded by:
\begin{equation}
	\label{s9}
    \begin{alignedat}{2}
    	s_9
        &=  \V (\partial G^{k-1})^T \nabla F(\hat{G}^{k-1}) \V^2\\
        &\leq \V \partial G^{k-1} \V^2 \V \nabla F(\hat{G}^{k-1}) - \nabla F(G^*) \V^2\\
        &\leq B_G^2 L_F^2 \frac{T_1}{m}.
    \end{alignedat}
\end{equation}
Similar to $s_7$, we can bound $s_{10}$ by:
\begin{equation}
	\label{s10}
    \begin{alignedat}{2}
    	s_{10}
        &= \V (\partial G^{k-1})^T (w^k)^T\left[ F'(\hat{G}^{k-1}) - Q^{k-1} \right] \V^2\\
        &\leq \V \partial G^{k-1} \V^2 \V w^k \V^2 \V F'(\hat{G}^{k-1}) - Q^{k-1} \V^2\\
        &\leq B_G^2 \V w^k \V^2 \left( n L_F^2 \frac{T_1}{m}  + T_0 \right).
    \end{alignedat}
\end{equation}
To bound $s_{11}$, observe that:
\begin{equation}
	\label{s11}
    \begin{alignedat}{2}
    	s_{11}
        &= (1-\frac{1}{n})^2 \V \left[G'^{k-1} - J^{k-1}\right]^T (z^k) e^T Q^{k-1} \V^2\\
        &\leq (1-\frac{1}{n})^2 \V G'^{k-1} - J^{k-1} \V^2 \V z^k \V^2 \V \sum_{i=1}^n Q_i^{k-1} - \nabla F_i(G^*) \V^2\\
        &\leq (1-\frac{1}{n})^2 2(\V G'^{k-1}\V^2 + \V J^{k-1} \V^2) \V z^k \V^2 n T_0\\
        &\leq (1-\frac{1}{n})^2 4m B_G^2 \V z^k \V^2 n T_0.
    \end{alignedat}
\end{equation}
Similarly, we can bound the term $s_{12}$ by:
	\begin{equation}
    \label{s12}
    \begin{alignedat}{2}
    	s_{12}
        &= \V \left[G'^{k-1} - J^{k-1}\right]^T (z^k) \nabla F(\hat{G}^{k-1}) \V^2\\
        &\leq \V G'^{k-1} - J^{k-1}\V^2 \V (z^k) \V^2 \V \nabla F(\hat{G}^{k-1}) \V^2\\
        &\leq 4m B_G^2 \V z^k \V^2 \V \nabla F(\hat{G}^{k-1}) - \nabla F(G^*) \V^2\\
        &\leq 4m B_G^2 \V z^k \V^2 L_F^2 \V \hat{G}^{k-1} - G^* \V^2\\
        &=4m B_G^2 \V z^k \V^2 L_F^2 \frac{T_1}{m}.
    \end{alignedat}
\end{equation}
Lastly, to bound $s_{13}$, proceeding similar to the cases of $s_{10}$ and $s_{12}$, we obtain:
\begin{equation}
	\label{s13}
    \begin{alignedat}{2}
    	s_{13}
        &= \V \left[G'^{k-1} - J^{k-1}\right]^T (z^k) (w^k)^T\left[ F'(\hat{G}^{k-1}) - Q^{k-1} \right] \V^2\\
        &\leq 4m B_G^2 \V z^k \V^2 \V w^k \V^2 \left( n L_F^2 \frac{T_1}{m}  + T_0 \right).
    \end{alignedat}
\end{equation}
Additionally, since $w_i^k$ and $w_{i'}^k$ are dependent for $i,i'\in \left[1,\ldots, n \right]$, we have the following,
\begin{equation}
	\begin{alignedat}{2}
		\E \V w \V^2 
        &= \E \V \left[ w_1I, w_2I, \ldots, w_nI \right]^T \V^2\\
        &=\E \left( w_1^2+w_2^2+\ldots+w_n^2 \right) \V I \V^2\\
        &=(1-\frac{1}{n})^2 + (-\frac{1}{n})^2(n-1)\\
        &=1-\frac{1}{n}.
	\end{alignedat}
\end{equation}
In a similar fashion, we can show that $\E \V z \V^2 = 1-\frac{1}{m} $. Hence, we can bound $\E \V S \V^2$ by:
\begin{equation}
	\label{boundSsquare}
	\begin{alignedat}{2}
		\E \V S \V^2
        &\leq 9&&\left( (m-1)^2(1-\frac{1}{n})n + 2(m-1)^2(1-\frac{1}{n}) \right.\\
        &  &&+ (1-\frac{1}{n})^2n + (1-\frac{1}{n}) + (1-\frac{1}{n})^24m  (1-\frac{1}{m})n\\
        & &&\left.+ 4m (1-\frac{1}{m})(1-\frac{1}{n})\right)B_G^2T_0\\
        &+ 9&&\left( \frac{(m-1)^2}{m} + 2(m-1)^2(1-\frac{1}{n})\frac{n}{m} \right.\\
        & &&+ \frac{1}{m} + (1-\frac{1}{n})\frac{n}{m} +4(1-\frac{1}{m}) \\
        & &&\left.+ 4 (1-\frac{1}{m})(1-\frac{1}{n})n \right)B_G^2L_F^2T_1\\
        &\leq 9&&\left( (m-1)^2(1-\frac{1}{n})(n + 2) + (n-1)(4m-3)\right)B_G^2T_0\\ 
        &+ \frac{9}{m}&&\left((m-1)^2(2n-1)+n+4n(m-1) \right)B_G^2L_F^2T_1\\
	\end{alignedat}
\end{equation}
Substituting Eq.(\ref{boundSsquare}) and Eq.(\ref{boundInnerProd}) into Eq.(\ref{decomp}), and choosing the parameters $\alpha_1$, $\alpha_2$, and $\alpha3$, appropriately, we have:
\begin{equation}
	\label{boundWT0T1}
    \begin{alignedat}{3}
    	&\E \V x^k - x^* \V^2 && \leq \V x^{k-1} - x^* \V^2 - 2\alpha \left( \frac{1}{2}\mu_f - \frac{16(m-a)}{m\mu_f}B_G^4 L_F^2\right)\V x^{k-1} -x^* \V^2\\
        & &&
        \begin{alignedat}{2}
        	+&\left[9\alpha^2\left((m-1)^2(1-\frac{1}{n})(n + 2) + (n-1)(4m-3)\right)\right.\\
            &\left. +16\alpha n \frac{(\beta_1^2 m^2 + \beta_3^2)}{\mu_f} \right]B_G^2T_0
        \end{alignedat}\\
        & &&
        \begin{alignedat}{2}
        	+&\left[\frac{9\alpha^2}{m}\left( (m-1)^2(2n-1)+n+4n(m-1) \right)\right.\\
            &\left. +16\alpha \frac{(\beta_2^2 m + 16(m-a))}{m^2\mu_f} \right]B_G^2L_F^2T_1.
        \end{alignedat}
    \end{alignedat}
\end{equation}

We now proceed to bound the remaining terms $T_0$ and $T_1$.
Starting with $T_0$,
\begin{equation}
	T_0 = \V Q^{k-1} = F'(G(x^*))\V^2 - \sum_{i=1}^n \V Q_i^{k-1} - \nabla F_i(G(x^*)) \V^2.
\end{equation}
We argue that although each  $Q_i$'s is not updated at each iteration, it must be updated at least once in $K$ iterations during the refresh. If we refer to the coefficient vector $x$ at the refresh stage as $\tilde{x}$, we have:
\begin{equation}
	\label{boundQ}
	\begin{alignedat}{2}
		\E \V Q_i(x^{k-1}) - \nabla F_i(G(x^*)) &\V^2
        && \leq \E \V F_i(G(\tilde{x})) - \nabla F_i(G(x^*)) \V^2\\
        & && \leq L_F^2 \V G(\tilde{x}) - G(x^*) \V\\
        & && \leq L_F^2 B_G^2 \V \tilde{x} - x^* \V^2\\
        \Rightarrow \blank &  T_0&&\leq n L_F^2 B_G^2 \V \tilde{x} - x^* \V^2.
	\end{alignedat}
\end{equation}
Next we bound $T_1$.
\begin{equation}
	\label{boundT1}
	\begin{alignedat}{2}
		\E[T_1]
        & = \E \sum_{j=1}^m \V V_j^{k} - G_j^* \V^2\\
        & = \sum_{j=1}^m \E \V V_j^k - G_j^* \V^2\\
        & = \sum_{j=1}^m \underset{t_1}{\E \underbrace{\V (1-\frac{a}{m})V_j^{k-1} +\frac{a}{m}G_j^{k-1} + y_j^k\left[ G_j^{k-1} - V_j^{k-1} \right] - G_j^* \V^2}}
	\end{alignedat}
\end{equation}
\begin{equation}
	\label{boundt1}
	\begin{alignedat}{2}
		\E[t_1]
        & = \E \V (1-\frac{a}{m})V_j^{k-1} +\frac{a}{m}G_j^{k-1} + y_j^k\left[ G_j^{k-1} - V_j^{k-1} \right] - G_j^* \V^2\\
        & = \E \V (1-\frac{a}{m})(V_j^{k-1} - G_j^*) + \frac{a}{m} \V (G_j^{k-1} - G_j^*) \V^2 + y_j^k\left[ G_j^{k-1} - V_j^{k-1} \right]\V^2\\
        &\leq 3\left( (1-\frac{a}{m})^2 \V V_j^{k-1} - G_j^* \V^2 + (\frac{a}{m})^2 \V (G_j^{k-1} - G_j^*) \V^2 + \E[(y_j^k)^2] \V G_j^{k-1} - V_j^{k-1} \V^2\right)\\
        &\leq 3\left( (1-\frac{a}{m})^2 \V V_j^{k-1} - G_j^* \V^2 + (\frac{a}{m})^2 B_G^2 \V (x^{k-1} - x^*) \V^2 + \E[(y_j^k)^2] \V G_j^{k-1} - V_j^{k-1} \V^2\right)\\
	\end{alignedat}
\end{equation}
Rewriting $(G_j^{k-1} - V_j^{k-1})$ as $(G_j^{k-1} - G_j^* + G_j^* - V_j^{k-1})$, we have:
\begin{equation}
	\label{boundt1subpart}
	\begin{alignedat}{2}
		&(1-\frac{a}{m})^2 \V V_j^{k-1} - G_j^* \V^2 + \E[(y_j^k)^2] \V G_j^{k-1} - V_j^{k-1} \V^2\\
        = &(1-\frac{a}{m})^2 \V V_j^{k-1} - G_j^* \V^2 + (1-\frac{a}{m})\frac{a}{m} \V G_j^{k-1} - G_j^* + G_j^* - V_j^{k-1} \V^2\\
        \leq &(1-\frac{a}{m})^2 \V V_j^{k-1} - G_j^* \V^2 + 2(1-\frac{a}{m})\frac{a}{m}\left(\V G_j^{k-1} - G_j^* \V^2 + \V V_j^{k-1} - G_j^* \V^2 \right)\\
        \leq &(1-\frac{a}{m})^2 \V V_j^{k-1} - G_j^* \V^2 + 2(1-\frac{a}{m})\frac{a}{m}\left(B_G^2 \V x^{k-1} - x^* \V^2 + \V V_j^{k-1} - G_j^* \V^2 \right)\\
        = &(1-\frac{a}{m})(1+\frac{a}{m}) \V V_j^{k-1} - G_j^* \V^2 + 2(1-\frac{a}{m})\frac{a}{m}B_G^2 \V x^{k-1} - x^* \V^2.
	\end{alignedat}
\end{equation}
Using an argument similar to that used in Eq.(\ref{boundQ}), we can bound $\V V_j^{k-1} - G_j^* \V^2$ using the coefficient vector $\tilde{x}$ at the refresh step:
\begin{equation}
	\label{boundV}
	\begin{alignedat}{2}
		\V V_j^{k-1} - G_j^* \V^2
        \leq \V G_j(\tilde{x}) - G_j(x^*) \V^2 \leq B_G^2 \V \tilde{x}-x^* \V^2.
	\end{alignedat}
\end{equation}
Substituting Eq.(\ref{boundV}) and Eq.(\ref{boundt1subpart}) back into Eq.(\ref{boundt1}), we obtain:
\begin{equation}
	\E[t_1] \leq 3\left( (1-\frac{a}{m})(1+\frac{a}{m}) B_G^2\V \tilde{x} - x^* \V^2 + (\frac{a}{m})(2-\frac{a}{m})B_G^2\V x^{k-1}-x^* \V^2\right).
\end{equation}
Finally, $\E[T_1]$ in Eq.(\ref{boundT1}) is bounded by:
\begin{equation}
	\label{finalBoundT1}
    \begin{alignedat}{2}
    	\E[T_1] &= \sum_{j=1}^m \E[t_1] \\
        &\leq 3m\left( (1-(\frac{a}{m})^2) B_G^2\V \tilde{x} - x^* \V^2 + (\frac{a}{m})(2-\frac{a}{m})B_G^2\V x^{k-1}-x^* \V^2\right).
    \end{alignedat}
\end{equation}
Substituting the above results back into Eq.(\ref{boundWT0T1}) we obtain:
\begin{equation}
	\label{finalbound}
	\begin{alignedat}{3}
    	&\E \V x^k - x^* \V^2 && \leq \V x^{k-1} - x^* \V^2 - 2\alpha \left( \frac{1}{2}\mu_f - \frac{16(m-a)}{m\mu_f}B_G^4 L_F^2\right)\V x^{k-1} -x^* \V^2\\
        & &&
        \begin{alignedat}{2}
        	+&\left[9\alpha^2\left((m-1)^2(1-\frac{1}{n})(n + 2) + (n-1)(4m-3)\right)\right.\\
            &\left. +16\alpha n \frac{(\beta_1^2 m^2 + \beta_3^2)}{\mu_f} \right]n B_G^4 L_F^2\V \tilde{x} - x^* \V^2
        \end{alignedat}\\
        & &&
        \begin{alignedat}{2}
        	+&\left[\frac{9\alpha^2}{m}\left( (m-1)^2(2n-1)+n+4n(m-1) \right)\right.\\
            &\left. +16\alpha \frac{(\beta_2^2 m + 16(m-a))}{m^2\mu_f} \right]B_G^4 L_F^2 3m \left( (1-(\frac{a}{m})^2)\V \tilde{x} - x^* \V^2 + (\frac{a}{m})(2-\frac{a}{m})\V x^{k-1}-x^* \V^2\right)
        \end{alignedat}\\
        & &&=\E \V x^{k-1} - x^* \V^2\\
        & && -\left[ \alpha \mu_f - \left( \frac{32\alpha (m-a)}{m\mu_f} + 3a(2-\frac{a}{m}) \sigma_2 \right) B_G^4 L_F^2\right]\V x^{k-1} - x^* \V^2\\
        & &&+\left[ n\sigma_1 + 3m \sigma_2\left( 1-(\frac{a}{m})^2 \right) \right]B_G^4 L_F^2\V \tilde{x} - x^* \V^2.
    \end{alignedat}
\end{equation}
Summing this inequality from $k=0$ (where $x^0 = \tilde{x}$) to $k=K-1$, we obtain:
\begin{equation}
	\begin{alignedat}{2}
		&\E \V x^K - x^* \V^2\\
        \leq \blank & \E \V \tilde{x} - x^* \V^2\\
        &-\left[ \alpha \mu_f - \left( \frac{32\alpha (m-a)}{m\mu_f} + 3a(2-\frac{a}{m}) \sigma_2 \right) B_G^4 L_F^2\right] \sum_{k=0}^{K-1} \E \V x^{k-1} - x^* \V^2\\
        &+K\left[ n\sigma_1 + 3m\sigma_2\left( 1-(\frac{a}{m})^2 \right) \right]B_G^4 L_F^2\V \tilde{x} - x^* \V^2,
	\end{alignedat}
\end{equation}
where
\begin{equation*}
	\begin{alignedat}{2}
    	&\begin{alignedat}{2}
        	\sigma_1 = &9\alpha^2\left((m-1)^2(1-\frac{1}{n})(n + 2) + (n-1)(4m-3)\right)+16\alpha n \frac{(\beta_1^2 m^2 + \beta_3^2)}{\mu_f} \\
        \end{alignedat}\\
		&\begin{alignedat}{2}
        	\sigma_2 = &\frac{9\alpha^2}{m}\left( (m-1)^2(2n-1)+n+4n(m-1) \right)+16\alpha \frac{(\beta_2^2 m + 16(m-a))}{m^2\mu_f}.
        \end{alignedat}
	\end{alignedat}
\end{equation*}
Hence, we have:
\begin{equation}
	\label{convbound}
	\begin{alignedat}{2}
    	&\frac{1}{K}\sum_{k=0}^{K-1}\E \V x_k - x^* \V^2 < \frac{\frac{1}{K}+\left( n\sigma_1 + 3m\left( 1-(\frac{a}{m})^2 \right)\sigma_2 \right)B_G^4 L_F^2}{ \alpha \mu_f - \left( \frac{32\alpha (m-a)}{m\mu_f} + 3a(2-\frac{a}{m}) \sigma_2 \right) B_G^4 L_F^2} \E \V \tilde{x}-x^* \V^2.
	\end{alignedat}
\end{equation}

In order to ensure that the fraction in Eq.(\ref{convbound}) above is  $< 1$, we need to choose the parameters ($a,K,\alpha$) appropriately:
\begin{alignat*}{2}
	&\frac{32\alpha (m-a)}{m\mu_f} B_G^4 L_F^2 < \frac{\alpha \mu_f}{4}\\
    \Rightarrow \blank & a > m \left( 1-\frac{\mu_f^2}{128 B_G^4 L_F^2} \right).
\end{alignat*}
Next, consider:
\begin{alignat*}{2}
	&3a\left( 2-\frac{a}{m}\right) \sigma_2  B_G^4 L_F^2 < \frac{\alpha \mu_f}{4}\\
    \Rightarrow \blank & \sigma_2 < \frac{\alpha \mu_f}{12a\left( 2-\frac{a}{m} \right) B_G^4 L_F^2}\\
    \Rightarrow \blank & \alpha < \frac{m\left( \frac{\mu_f}{12a\left( 2-\frac{a}{m} \right) B_G^4 L_F^2} - 16\frac{(\beta_2^2 m + 16(m-a))}{m^2\mu_f}\right)}{9\left( (m-1)^2(2n-1)+n+4n(m-1) \right)} = \alpha_1.
\end{alignat*}
Hence, we have:
\begin{alignat*}{2}
	&\frac{1}{K}\sum_{k=0}^{K-1}\E \V x_k - x^* \V^2
        < \frac{\frac{1}{K}+\left( n\sigma_1 + 3m\left( 1-(\frac{a}{m})^2 \right)\sigma_2 \right)B_G^4 L_F^2}{ \frac{\alpha \mu_f}{2} } \E \V \tilde{x}-x^* \V^2
\end{alignat*}
Next consider:
\begin{alignat*}{2}
	&n\sigma_1 B_G^4 L_F^2 < \frac{\alpha \mu_f}{8}\\
    \Rightarrow \blank &\sigma_1 < \frac{\alpha \mu_f}{8n B_G^4 L_F^2}\\
    \Rightarrow \blank &\alpha < \frac{\frac{\mu_f}{8n B_G^4 L_F^2} - 16 n \frac{(\beta_1^2 m^2 + \beta_3^2)}{\mu_f}}{9\left((m-1)^2(1-\frac{1}{n})(n + 2) + (n-1)(4m-3)\right)} = \alpha_2,
\end{alignat*}
and
\begin{alignat*}{2}
	&3m\left( 1-\left( \frac{a}{m} \right)^2 \right)\sigma_2 B_G^4 L_F^2 < \frac{\alpha \mu_f}{8}\\
    \Rightarrow \blank &\sigma_2 < \frac{\alpha \mu_f}{24m\left( 1-\left( \frac{a}{m} \right)^2 \right)B_G^4 L_F^2}\\
    \Rightarrow \blank &\alpha < \frac{m \left( \frac{\mu_f}{24m\left( 1-\left( \frac{a}{m} \right)^2 \right)B_G^4 L_F^2} - 16 \frac{(\beta_2^2 m + 16(m-a))}{m^2\mu_f} \right)}{9\left( (m-1)^2(2n-1)+n+4n(m-1) \right)} = \alpha_3,
\end{alignat*}
Hence, We have:
\begin{alignat*}{2}
	&\frac{1}{K}\sum_{k=0}^{K-1}\E \V x_k - x^* \V^2 < \frac{\frac{1}{K}+\frac{\alpha \mu_f}{4}}{ \frac{\alpha \mu_f}{2} } \E \V \tilde{x}-x^* \V^2
\end{alignat*}
Thus, if
\begin{alignat*}{2}
	&\frac{1}{K} < \frac{\alpha \mu_f}{8}\\
    \Rightarrow \blank & K > \frac{8}{\alpha \mu_f},
\end{alignat*}
we have:
\begin{alignat*}{2}
	&\frac{1}{K}\sum_{k=0}^{K-1}\E \V x_k - x^* \V^2 < \frac{\frac{\alpha \mu_f}{8}+\frac{\alpha \mu_f}{4}}{ \frac{\alpha \mu_f}{2} } \E \V \tilde{x}-x^* \V^2 = \frac{3}{4} \E \V \tilde{x}-x^* \V^2.
\end{alignat*}
Thus, if
\begin{alignat*}{2}
	\begin{cases}
		& a > m \left( 1-\frac{\mu_f^2}{128 B_G^4 L_F^2} \right),\\
    	& \alpha < \min \left\lbrace \alpha_1,\alpha_2,\alpha_3 \right\rbrace,\\
    	& K > \frac{8}{\alpha \mu_f},
	\end{cases}
\end{alignat*}
then C-SAG converges linearly with a rate of $3/4$. This concludes the proof.
\end{document}